\documentclass[preprint,12pt,authoryear]{elsarticle}

\usepackage{import}     
\usepackage{graphicx}   
\usepackage{url}        
\usepackage{comment}    
\usepackage{pdfpages}   
\usepackage{float}      
\usepackage{wrapfig}    

\usepackage{caption}        
\usepackage{subcaption}     
\usepackage[nottoc]{tocbibind}  

\usepackage{bm}         
\usepackage{bbm}        
\usepackage{amsmath}    
\usepackage{amssymb}    
\usepackage{amsthm}     
\usepackage{textcomp}   
\usepackage{gensymb}    

\usepackage{siunitx}    
\usepackage[british]{babel}     
\addto\extrasbritish{   
    \def\sectionautorefname{Section}                
}

\usepackage{tikz}
\usetikzlibrary{shapes,arrows}
\tikzstyle{decision} = [diamond, draw, fill=blue!20, 
    text width=4.5em, text badly centered, node distance=3cm, inner sep=0pt]
\tikzstyle{block} = [rectangle, draw, fill=blue!20, 
    text width=5em, text centered, rounded corners, minimum height=4em]
\tikzstyle{line} = [draw, -latex']
\tikzstyle{cloud} = [draw, ellipse,fill=red!20, node distance=3cm,
    minimum height=2em]

\usepackage{hyperref}   
\hypersetup{    
    colorlinks,
    citecolor = black,
    filecolor = black,
    linkcolor = black,
    urlcolor = black
}

\input{setup.sty}

\usepackage{natbib}  

\DeclareMathOperator*{\argmax}{argmax}

\newcommand{\Cov}{\operatorname{Cov}}

\def\E{\mathbb{E}}

\def\eps{\epsilon}

\def\Var{\mathop{\mathrm{Var}}}

\def\ba{\boldsymbol{a}}

\def\be{\boldsymbol{e}}
\def\bu{\boldsymbol{u}}
\def\bx{\boldsymbol{x}}
\def\by{\boldsymbol{y}}
\def\bp{\boldsymbol{p}}

\def\bbeta{\boldsymbol{\beta}}

\def\balpha{\boldsymbol{\alpha}}

\def\btheta{\boldsymbol{\theta}}

\def\bpi{\boldsymbol{\pi}}

\def\bzero{\boldsymbol{0}}
\def\bSigma{\boldsymbol{\Sigma}}

\def\bmu{\boldsymbol{\mu}}

\def\bX{\mathbf{X}}

\def\bS{\mathbf{S}}
\def\bU{\mathbf{U}}
\def\bA{\mathbf{A}}
\def\bB{\mathbf{B}}

\journal{Computational Statistics \& Data Analysis}

\begin{document}

\begin{frontmatter}

\title{Fairness Constraints in High-Dimensional Generalized Linear Models}

\author[inst1]{Yixiao Lin\corref{cor1}}
\ead{yl2883@cornell.edu}
\author[inst1]{James G. Booth}
\ead{jb383@cornell.edu}

\cortext[cor1]{Corresponding author.}
\address[inst1]{Department of Statistics and Data Science, Cornell University}

\begin{abstract}
Most fairness-aware learning methods assume that sensitive attributes are observed,
an assumption that may fail because of privacy, legal, or data-collection constraints.
We develop a framework for fairness-aware generalized linear models when the sensitive
attribute is latent, possibly multi-category, and the predictors may be high-dimensional.
Candidate proxy variables are modeled using Gaussian mixtures for continuous predictors
and product-multinomial mixtures for categorical predictors. The resulting posterior
group-membership probabilities are used to residualize the predictors and reduce their
association with the latent sensitive structure. For continuous outcomes, we use constrained least squares to limit the contribution of
the estimated sensitive attribute to prediction variability. For binary outcomes, we use
penalized logistic regression to reduce dependence between predicted probabilities and
estimated group membership. In high-dimensional settings, the procedure is combined
with SEMMS variable selection to obtain sparse models. We establish identifiability and latent-class recovery results for Gaussian and categorical
mixtures and derive expressions that quantify the loss of predictive power after
residualization.Simulations and real-data applications demonstrate favorable accuracy-fairness performance.
\end{abstract}

\end{frontmatter}

\section{Introduction}
\label{sec:Intro}

\subsection{Motivation}

Algorithmic decision-making pervades our lives, profoundly shaping outcomes in
areas such as email filtering, image tagging, personalized news feeds, credit
scoring, and job assessments. While these systems offer efficiency and
convenience, they also raise significant concerns regarding transparency,
accountability, and fairness. These concerns are especially pressing when the
historical data used to train such systems reflect biases associated with
sensitive attributes, such as gender, race, or religion. When unchecked, these
biases can lead to systematically unfair treatment of certain groups and
perpetuate existing societal inequities.

A large body of work in fair machine learning seeks to mitigate these risks by
constraining or modifying learning algorithms so that predictions satisfy
group-based fairness notions, such as demographic parity or equalized odds.
However, two practical challenges remain relatively underexplored. First, most
methods assume that sensitive attributes are directly observed and available
throughout the development and deployment of a system. In many real-world
applications, this assumption fails due to privacy concerns, ethical
considerations, or legal restrictions~\citep{coston2019missingprotected}. Second,
modern applications often involve high-dimensional feature spaces in which
only a small subset of predictors is truly informative. Variable selection is
therefore essential; yet, the interaction between feature selection, fairness
constraints, and unobserved sensitive attributes is poorly understood.

In this paper, we address these two challenges simultaneously. We consider
settings in which sensitive attributes are multi-category, unobserved, and
must be inferred from auxiliary predictors, while the main prediction task is
modeled using generalized linear models (GLMs) in either low- or
high-dimensional regimes. Our goal is to construct prediction rules that are
both accurate and fair with respect to the latent sensitive attribute, while
retaining the interpretability and sparsity benefits of modern variable
selection methods.

\subsection{Related work}
Existing approaches to algorithmic fairness can be broadly categorized into three classes.

\begin{enumerate}
    \item \textbf{Preprocessing approaches.}
    These methods modify the training data before model fitting in order to mitigate discriminatory patterns, for example, through reweighting, resampling, or data transformation to balance the representation of sensitive groups \citep{feldman2015disparate,kamiran2012preprocessing}.
    \item \textbf{In-processing approaches.}
    These methods incorporate fairness constraints, regularization terms, or other fairness-aware objectives directly into the learning procedure so that model training jointly optimizes predictive performance and fairness criteria \citep{zafar2019flexible,agarwal2018reductions,kamishima2012prejudice,kleindessner2022pairwise}.
    \item \textbf{Post-processing approaches.}
    These methods adjust the outputs of an already trained model to satisfy fairness criteria, which is especially useful when retraining the model or modifying the original learning algorithm is infeasible \citep{hardt2016equality,pleiss2017calibration}.
\end{enumerate}

Despite their diversity, most of these methods rely on access to sensitive attributes, either to define fairness criteria or to implement interventions. When such attributes are unavailable, fairness becomes substantially more challenging. Existing work has considered fairness under missing protected attributes, proxy-sensitive features, or latent demographic structure \citep{coston2019missingprotected,zhu2022learning,zhao2022fairrf,lahoti2020arl,yan2020faircb,grari2022causalvae,ni2024knowledge}. For instance, \citet{zhu2022learning} proposes a probabilistic graphical model in which a variational autoencoder estimates latent sensitive attributes from relevant features, and fairness interventions are then applied to the learned representation. While this proxy-based strategy is promising, it typically assumes prior knowledge of which features are relevant for recovering the sensitive attribute, an assumption that may be unrealistic in complex applications with many covariates and intricate dependency structures. More broadly, different fairness notions may be incompatible with one another or with predictive calibration in general settings \citep{kleinberg2017tradeoffs,pleiss2017calibration,chouldechova2017fairprediction}.

Fairness in high-dimensional settings, particularly when combined with
variable selection, has also received limited attention. \citet{tarzanagh2021fair}
studies fairness-aware clustering and introduces $\ell_1$-regularization to
promote sparsity in feature selection. However, analogous developments for
supervised learning—such as regression or classification with fairness
constraints and automatic variable selection—are far less mature.
High-dimensional predictors exacerbate the difficulty of ensuring fairness:
not only may many features carry implicit information about group membership,
but selecting a small subset of predictors can itself introduce or amplify
disparities if it preferentially retains group-informative variables.
For broader overviews of fairness in machine learning, see \citet{mehrabi2021survey} and \citet{corbettdavies2023measure}.
\subsection{Contributions}

In response to these challenges, we develop a unified framework for
fairness-aware generalized linear models when multi-category sensitive
attributes are unobserved and predictors may be high-dimensional. Our main
contributions are threefold.

\begin{enumerate}
    \item \textbf{Latent sensitive attribute estimation via mixture models.}
    Our first contribution is to model the unobserved sensitive attribute through a latent mixture structure, estimated via an EM-type procedure \citep{dempster1977em,mclachlan2000finite}. We assume that a subset of predictors depends on an unobserved
    multi-category sensitive attribute and model this dependence using either
    Gaussian mixture models or mixtures of product-multinomials. We derive
    EM-based estimators for the corresponding parameters and use the posterior
    class probabilities as estimates of the sensitive attribute. To improve
    robustness when many predictors are only weakly related to the sensitive
    attribute, we also propose a simple screening strategy that estimates the
    sensitive attribute from each candidate predictor separately and selects the estimate that yields the best
    predictive performance or fairness according to an appropriate criterion.
    
    \item \textbf{Fairness-aware GLMs with variable selection.} Given the
    estimated sensitive attribute, we regress the predictors on its posterior
    expectation and construct residuals that are less informative about group
    membership. We then fit fairness-aware GLMs using these residuals: a
    constrained least-squares formulation for continuous responses that limits
    the coefficient of determination obtained by regressing predictions on the
    estimated sensitive attribute, and a penalized logistic regression
    formulation for binary responses that incorporates a fairness penalty
    based on centered correlations. In high-dimensional settings, we integrate
    this fairness machinery with the SEMMS empirical Bayes variable selection
    procedure \citep{bar2020semms}, yielding sparse, fairness-aware GLMs that
    automatically select a small set of predictive features.
    
    \item \textbf{Theoretical and empirical characterization of
    accuracy-fairness trade-offs.} Compared with existing fairness-aware learning approaches \citep{komiyama2018nonconvex,zafar2019flexible,zhao2022fairrf}, our method addresses the practically important setting in which the sensitive attribute is not directly observed. We study the conditions under which the
    mixture models for the predictors are identifiable and derive explicit
    expressions and bounds for the probability of correctly recovering the
    sensitive attribute under both Gaussian and categorical mixtures. For
    regression, we obtain closed-form expressions for the coefficients of
    determination before and after regressing out the sensitive attribute,
    thereby quantifying the loss of predictive power induced by fairness
    constraints and characterizing the fundamental accuracy-fairness
    trade-off in our framework. Through extensive simulations and applications
    to benchmark datasets, including ADULT, COMPAS, and ARRHYTHMIA, we validate
    these theoretical insights and show that our approach can substantially
    reduce group disparities with only a modest loss in predictive accuracy.
\end{enumerate}

\subsection{Notation and problem setup}
\label{subsec:intro-notation}

We now formalize the problem setup and introduce the notation used throughout the
paper. Each $p$-dimensional vector is treated as a column vector and identified
with a $p \times 1$ matrix. Let $n$ denote the number of observations. The
$i$th observation is denoted by $(\bx_i, y_i)$, where $\bx_i \in \mathbb{R}^p$
is the $p$-dimensional feature vector and $y_i$ is the response. The response
can be either continuous or binary, i.e. $y_i \in \mathbb{R}$ or
$y_i \in \{0,1\}$. Let
\[
\by = (y_1,\dots,y_n)^\top \in \mathbb{R}^{n}
\quad\text{and}\quad
\bX = [\bx_1,\dots,\bx_n]^\top \in \mathbb{R}^{n\times p}
\]
denote the response vector and feature matrix, respectively.

We assume the existence of an unobserved multi-category sensitive attribute
$A_i \in \{1,\dots,K\}$ for each observation. For convenience, we also use 
one-hot encoding
\[
\ba_i = (a_{i1},\dots,a_{iK})^\top \in \{0,1\}^K, \qquad
\sum_{k=1}^K a_{ik} = 1,
\]
and collect these into the $n \times K$ matrix
$\bA = [\ba_1,\dots,\ba_n]^\top$. The sensitive attribute $A_i$ is not
observed, but we assume that a subset of predictors carries information about
it. Specifically, we partition the columns of $\bX$ as
$\bX = [\bX^{\ba}, \bX^{\bar{\ba}}]$, where $\bX^{\ba}$ consists of predictors
whose distribution depends on the sensitive attribute and
$\bX^{\bar{\ba}}$ consists of predictors that are conditionally independent
of $A_i$.

Our main prediction model is a generalized linear model. Let
$\mu_i = \E(y_i \mid \bx_i, \bbeta)$ and $g(\cdot)$ denote a link function. Then we consider linear predictors of the form
\begin{equation}
    \label{eq:linear_predictor}
    g(\mu_i) = \eta_i = \sum_{j=1}^p \beta_j x_{ij},
\end{equation}
where $\bbeta = (\beta_1,\dots,\beta_p)^\top$ is a vector of regression
coefficients. We study two regimes:
\begin{itemize}
    \item a \emph{low-dimensional} setting in which $\bbeta$ is an unknown but
    fixed parameter vector and $p < n$; and
    \item a \emph{high-dimensional} setting in which $p$ may exceed $n$ and
    sparsity is induced through an empirical Bayes prior of the form
    $\beta_j = u_j \gamma_j$, where $u_j \sim N(\mu_u,\sigma^2)$ and
    $\gamma_j \in \{-1,0,1\}$ with
    $\Pr(\gamma_j = g) = \pi_g$ for $g \in \{-1,0,1\}$. In this case, the
    $j$th predictor is effectively included in the model only if
    $\gamma_j \neq 0$, as determined by the SEMMS algorithm
    \citep{bar2020semms}.
\end{itemize}

To estimate the latent sensitive attribute, we model the rows of
$\bX^{\ba}$ as independent draws from a finite mixture distribution, the
components of which correspond to the levels of $A_i$. In particular, we assume that
\begin{equation}
    \label{eq:mixture_model}
    \bx_i^{\,\ba} \mid A_i = k \sim f_k(\cdot;\btheta_k),
    \qquad k = 1,\dots,K,
\end{equation}
where $f_k$ is either a multivariate Gaussian density with component-specific
means (and possibly covariances) or a product-multinomial model with
component-specific probability vectors. Writing
$\bp = (p_1,\dots,p_K)^\top$ for the mixing proportions, the marginal density
of $\bx_i^{\,\ba}$ is then
\[
    f(\bx_i^{\,\ba}) = \sum_{k=1}^K p_k f_k(\bx_i^{\,\ba};\btheta_k).
\]
We use the EM algorithm to obtain maximum likelihood estimates of
$\bp$ and $\{\btheta_k\}_{k=1}^K$, and we take the resulting posterior
probabilities
\[
    \hat a_{ik} = \Pr(A_i = k \mid \bx_i^{\,\ba}; \hat\bp, \hat\btheta_1,\dots,\hat\btheta_K),
    \qquad k = 1,\dots,K,
\]
as estimates of the entries of $\ba_i$. Collecting these into
$\hat\bA = [\hat\ba_1,\dots,\hat\ba_n]^\top$, we then regress $\bX$ on
$\hat\bA$ and compute the residual matrix
$\hat\bU = [\hat\bu_1,\dots,\hat\bu_n]^\top$. The final step is to solve an
optimization problem that balances fairness, as measured by the dependence of
predictions on $\hat\bA$, against predictive accuracy, as measured by the
ability of $\hat\bU$ to predict $\by$.

Each of these steps is explained in more detail in Sections~\ref{sec:method}
and~\ref{sec:theory}. Section~\ref{sec:simulations} presents results from
simulation studies, while Section~\ref{sec:realdata} demonstrates the approach
on real datasets. Section~\ref{sec:conclusion} concludes with a discussion of
our findings and directions for future work.

\section{Methodology}
\label{sec:method}

In this section, we describe our three-stage procedure: (i) estimating the
latent sensitive attribute from auxiliary predictors using mixture models;
(ii) constructing residualized predictors that are less informative about
group membership; and (iii) fitting fairness-aware generalized linear models
(GLMs), with optional variable selection via SEMMS in high-dimensional
settings.

\subsection{Estimating the latent sensitive attribute}
\label{subsec:latentA}


Given the parameters $(\bp,\btheta_1,\dots,\btheta_K)$, the posterior
probabilities for the latent classes are
\begin{equation}
    \label{eq:posterior-a}
    a_{ik} = \Pr(A_i = k \mid \bx_i^{\,\ba})
    = \frac{p_k f_k(\bx_i^{\,\ba};\btheta_k)}
           {\sum_{\ell=1}^K p_\ell f_\ell(\bx_i^{\,\ba};\btheta_\ell)},
    \qquad k = 1,\dots,K.
\end{equation}
We collect these into the soft group-membership vector
$\tilde\ba_i = (a_{i1},\dots,a_{iK})^\top$ and estimate it by substituting
maximum likelihood estimates for $(\bp,\btheta_1,\dots,\btheta_K)$ into
\eqref{eq:posterior-a}. Writing
$\hat\ba_i = (\hat a_{i1},\dots,\hat a_{iK})^\top$ for the resulting estimate
of $\tilde\ba_i$, we obtain the estimated sensitive attribute matrix
$\hat\bA = [\hat\ba_1,\dots,\hat\ba_n]^\top$. We estimate model parameters using the EM algorithm \citep{dempster1977em}, a standard approach for latent-variable and finite-mixture models \citep{mclachlan2000finite}.

\paragraph{EM algorithm.}
At
iteration $t+1$, the E-step computes the current posterior probabilities
\begin{equation}
    \label{eq:em-E}
    \hat a_{ik}^{(t+1)}
    = \frac{\hat p_k^{(t)} f_k(\bx_i^{\,\ba};\hat\btheta_k^{(t)})}
           {\sum_{\ell=1}^K \hat p_\ell^{(t)} f_\ell(\bx_i^{\,\ba};\hat\btheta_\ell^{(t)})},
    \qquad i = 1,\dots,n,\; k = 1,\dots,K,
\end{equation}
and the M-step updates the mixing proportions and component parameters via
\begin{align}
    \label{eq:em-M-p}
    \hat p_k^{(t+1)} 
    &= \frac{1}{n} \sum_{i=1}^n \hat a_{ik}^{(t+1)}, 
    \qquad k = 1,\dots,K, \\
    \label{eq:em-M-theta}
    \hat\btheta^{(t+1)}
    &= \argmax_{\btheta}
       \sum_{i=1}^n \sum_{k=1}^K 
       \hat a_{ik}^{(t+1)} 
       \log f_k(\bx_i^{\,\ba};\btheta_k),
\end{align}
where $\btheta = (\btheta_1,\dots,\btheta_K)$. The EM iterations are repeated
until convergence. We denote the final estimates by
$(\hat\bp,\hat\btheta_1,\dots,\hat\btheta_K)$ and set
\[
    \hat a_{ik}
    = \Pr\bigl(A_i = k \mid \bx_i^{\,\ba};\hat\bp,\hat\btheta_1,\dots,\hat\btheta_K\bigr)
\]
as in \eqref{eq:posterior-a}. Collecting the rows yields the estimated
sensitive attribute matrix $\hat\bA = [\hat\ba_1,\dots,\hat\ba_n]^\top$.

In what follows we consider two specific choices for the component densities
$f_k$: a Gaussian mixture model for continuous predictors and a product
multinomial mixture for categorical predictors. When the relevant predictors
contain both continuous and categorical variables, we use the natural hybrid
extension obtained by multiplying the Gaussian and categorical component
factors, together with the corresponding EM updates for each block.

\subsubsection{Gaussian mixtures}\label{subsubsec: Gaussian est}

When the predictors in $\bx_i^{\,\ba}$ are continuous, we use a Gaussian
mixture model with a shared covariance matrix. Specifically,
\[
    f_k(\bx_i^{\,\ba};\btheta_k) 
    = \phi(\bx_i^{\,\ba};\bmu_k,\bSigma)
    = |2\pi\bSigma|^{-1/2}
      \exp\left\{-\frac{1}{2}
      (\bx_i^{\,\ba} - \bmu_k)^\top\bSigma^{-1}
      (\bx_i^{\,\ba} - \bmu_k)\right\},
\]
where $\btheta_k = \bmu_k$ and $\bSigma$ are the common covariance matrix.
The M-step updates \eqref{eq:em-M-theta} simplify to
\begin{align}
    \label{eq:em-gaussian-mu}
    \hat\bmu_k^{(t+1)}
    &= \frac{\sum_{i=1}^n \hat a_{ik}^{(t+1)} \bx_i^{\,\ba}}
            {\sum_{i=1}^n \hat a_{ik}^{(t+1)}},
            \qquad k = 1,\dots,K, \\
    \label{eq:em-gaussian-Sigma}
    \hat\bSigma^{(t+1)}
    &= \frac{\sum_{i=1}^n \sum_{k=1}^K 
            \hat a_{ik}^{(t+1)}
            (\bx_i^{\,\ba} - \hat\bmu_k^{(t+1)})
            (\bx_i^{\,\ba} - \hat\bmu_k^{(t+1)})^\top}
            {\sum_{i=1}^n \sum_{k=1}^K \hat a_{ik}^{(t+1)}}.
\end{align}

\subsubsection{Categorical mixtures}\label{subsubsec: Categorical est}

When $\bx_i^{\,\ba}$ consists of $D$ independent categorical predictors with
$m_d$ levels for the $d$th predictor, we represent each predictor by
$(m_d-1)$ dummy variables and obtain a design matrix
$\bX^{\ba} \in \mathbb{R}^{n\times p_{\ba}}$ with
$p_{\ba} = \sum_{d=1}^D (m_d-1)$. For each component $k$, we let
$\bpi_{kd} = (\pi_{kd1},\dots,\pi_{kd m_d})^\top$ denote the probability
vector for the $d$th categorical predictor, with
$\sum_{\ell=1}^{m_d} \pi_{kd\ell} = 1$. The component distribution factorizes
as
\[
    f_k(\bx_i^{\,\ba};\btheta_k)
    = \prod_{d=1}^D \prod_{\ell=1}^{m_d}
      \pi_{kd\ell}^{\,z_{id\ell}},
\]
where $z_{id\ell} = 1$ if the $d$th categorical predictor of observation $i$
takes its $\ell$th level and $z_{id\ell} = 0$ otherwise. Here
$\btheta_k = \{\bpi_{kd}\}_{d=1}^D$.

The M-step update \eqref{eq:em-M-theta} yields the usual weighted
multinomial maximum likelihood estimator:
\begin{equation}
    \label{eq:em-multinomial-pi}
    \hat\pi_{kd\ell}^{(t+1)}
    = \frac{\sum_{i=1}^n \hat a_{ik}^{(t+1)} z_{id\ell}}
           {\sum_{i=1}^n \hat a_{ik}^{(t+1)}},
    \qquad k = 1,\dots,K,\; d = 1,\dots,D,\; \ell = 1,\dots,m_d.
\end{equation}

\subsection{Selecting relevant predictors for mixture modeling}
\label{subsec:screening}

In practice, estimating $\hat\bA$ using all predictors can be inaccurate when
many columns of $\bX$ are weakly related to the sensitive attribute or not
related at all. Including such predictors inflates variability without
improving the identifiability or estimation accuracy of the mixture model. To
mitigate this, we adopt a simple screening strategy.

For each candidate predictor (or small subset of predictors) $j$, we fit a
mixture model using only the corresponding columns $\bX_j^{\ba}$, compute
posterior probabilities $\hat\ba_i^{(j)}$, and obtain an estimate
$\hat\bA^{(j)}$. Using $\hat\bA^{(j)}$, we carry out the residualization and
fairness-aware modeling steps described in Sections~\ref{subsec:residual}
and~\ref{subsec:regression}--\ref{subsec:classification}, obtaining
predictions $\hat\by^{(j)}$ for the primary response.

In regression problems, we select the estimate $\hat\bA^{(j)}$ that maximizes
the correlation between $\hat\by^{(j)}$ and $\by$:
\[
    j^\star 
    = \argmax_j \operatorname{Cor}(\hat\by^{(j)}, \by).
\]
In binary classification problems, we instead use a group-wise mean distance
fairness measure. This is the screening criterion used in our real-data
classification experiments, because it tracks the downstream fairness notion
employed later in the penalized logistic model. For a given $j$, let
\[
    w_i^{(j)} = \max_{k} \hat a_{ik}^{(j)} - \frac{1}{K},
\]
and define the soft group $k$ as
\[
    S_k^{(j)} = \{ i : \hat a_{ik}^{(j)} 
        \ge \hat a_{i\ell}^{(j)} \text{ for all } \ell \neq k \}.
\]
We compute the maximum absolute difference between weighted group means:
\begin{equation}
    \label{eq:MD}
    \mathrm{MD}_j
    = \max_{1\le k,\ell\le K}
      \left|
      \frac{\sum_{i\in S_k^{(j)}} w_i^{(j)} y_i}
           {\sum_{i\in S_k^{(j)}} w_i^{(j)}}
      -
      \frac{\sum_{i\in S_\ell^{(j)}} w_i^{(j)} y_i}
           {\sum_{i\in S_\ell^{(j)}} w_i^{(j)}}
      \right|.
\end{equation}
Smaller values of $\mathrm{MD}_j$ indicate more similar performance across
the estimated groups. Depending on the application, we may choose $\hat\bA^{(j)}$
to either minimize $\mathrm{MD}_j$ (emphasizing fairness) or to balance it
against predictive accuracy.

In the sequel, we assume that a particular estimate $\hat\bA$ has been
selected by one of the above criteria and describe how to construct fair
predictive models conditioned on $\hat\bA$.

\subsection{Residualization}
\label{subsec:residual}

The residualization step is intended to reduce the influence of variables most strongly associated with the latent sensitive structure, in the spirit of removing discriminatory or proxy information before downstream prediction \citep{feldman2015disparate,kamiran2012preprocessing,zhao2022fairrf}.

Given the estimated sensitive attribute matrix $\hat\bA$, we seek predictors
that remain informative about $y$ but are less directly associated with group
membership. To this end, we regress the full feature matrix $\bX$ on
$\hat\bA$. Specifically, we consider the linear model
\[
    \bX = \mathbf{1}_n \bmu^\top + \hat\bA \bB + \bU,
\]
where $\mathbf{1}_n$ is the $n$-dimensional vector of ones, $\bmu$ is a
$p$-dimensional intercept vector, $\bB$ is a $K\times p$ matrix of regression
coefficients, and $\bU$ is an $n\times p$ matrix of residuals. We estimate
$(\bmu,\bB)$ by ordinary least squares, column-wise or jointly, and define the
residualized predictors as $\hat\bU = [\hat\bu_1,\dots,\hat\bu_n]^\top$.

Intuitively, $\hat\bU$ contains the variation in $\bX$ that is orthogonal (in a
least-squares sense) to the estimated sensitive attribute $\hat\bA$. By
design, linear functions of $\hat\bu_i$ are less correlated with $\hat\ba_i$
than linear functions of $\bx_i$, which makes them a natural basis for
fairness-aware modeling.

\subsection{Fairness-aware regression}
\label{subsec:regression}

Our constrained formulation is related to prior fairness-constrained learning approaches in regression and classification \citep{komiyama2018nonconvex,zafar2019flexible,agarwal2018reductions}.

We first consider the case of a continuous response $y_i \in \mathbb{R}$. Using
the estimated sensitive attribute $\hat\ba_i$ and residualized predictors
$\hat\bu_i$, we define the linear predictor
\begin{equation}
    \label{eq:linear-predictor-aug}
    \eta_i = \beta_0 + \hat\ba_i^\top \balpha + \hat\bu_i^\top \bbeta,
    \qquad i = 1,\dots,n,
\end{equation}
where $\beta_0 \in \mathbb{R}$ is an intercept and
$\balpha \in \mathbb{R}^K$, $\bbeta \in \mathbb{R}^p$ are coefficient vectors.
Because the entries of $\hat\ba_i$ sum to one, the intercept and all $K$ membership columns are linearly dependent. In implementation, we therefore omit one membership column (or, equivalently, impose a constraint such as $\mathbf{1}_K^\top\balpha=0$).
The corresponding predictions are $\hat y_i = \eta_i$.

Our fairness criterion is based on the empirical coefficient of
determination obtained by regressing $\hat\by = (\hat y_1,\dots,\hat y_n)^\top$
on $\hat\bA$. Let $\bS_{\hat\bA}$ and $\bS_{\hat\bU}$ denote the sample
covariance matrices of the rows of $\hat\bA$ and $\hat\bU$, respectively. It
can be shown (see \citep{komiyama2018nonconvex}) that the $R^2$ of the
regression of $\hat\by$ on $(\hat\bA,\hat\bU)$ and the proportion of variance
explained by $\hat\bA$ alone can be expressed in terms of
$\balpha^\top\bS_{\hat\bA}\balpha$ and $\bbeta^\top\bS_{\hat\bU}\bbeta$. We
enforce fairness by constraining the fraction of explained variance
attributable to $\hat\bA$ to be at most $\varepsilon \in [0,1]$.

Formally, we solve the constrained least-squares problem
\begin{equation}
    \label{eq:reg-constraint}
    \min_{\beta_0,\balpha,\bbeta}
        \sum_{i=1}^n (y_i - \hat y_i)^2
    \quad
    \text{subject to } R^2(\hat\by \mid \hat\bA) \le \varepsilon,
\end{equation}
where $R^2(\hat\by \mid \hat\bA)$ is the coefficient of determination of the
regression of $\hat\by$ on $\hat\bA$. After centering the variables, problem
\eqref{eq:reg-constraint} is equivalent to minimizing
\begin{equation}
    \label{eq:reg-objective}
    \balpha^\top \bS_{\hat\bA} \balpha
    + \bbeta^\top \bS_{\hat\bU} \bbeta
    - \frac{2}{n} \sum_{i=1}^n 
      y_i (\hat\ba_i^\top \balpha + \hat\bu_i^\top \bbeta)
\end{equation}
subject to the quadratic constraint
\begin{equation}
    \label{eq:R2-constraint}
    (1-\varepsilon)\,\balpha^\top \bS_{\hat\bA} \balpha
    - \varepsilon\,\bbeta^\top \bS_{\hat\bU} \bbeta \le 0.
\end{equation}
When $\varepsilon = 0$, the constraint \eqref{eq:R2-constraint} forces
$\balpha = \bzero$, which corresponds to a regression model that is fair with
respect to $\hat\bA$ (predictions depend only on $\hat\bu_i$). For
$\varepsilon$ close to $1$, the constraint becomes weak, and the solution
approximates the unconstrained least-squares estimator. Intermediate values of
$\varepsilon$ trace out an explicit accuracy–fairness trade-off.

In the low-dimensional setting ($p < n$), the problem
\eqref{eq:reg-objective}--\eqref{eq:R2-constraint} can be solved using standard
quadratically constrained quadratic programming. In the high-dimensional
setting ($p \gg n$), we combine this fairness formulation with the SEMMS
empirical Bayes variable selection prior described in
Section~\ref{subsec:intro-notation}, using the fairness constraint to guide
the selection toward predictors that are both predictive and less
group-informative.

\subsection{Fairness-aware binary classification}
\label{subsec:classification}

We now consider binary responses $y_i \in \{0,1\}$. Using the same linear
predictor as in \eqref{eq:linear-predictor-aug}, we define the success
probability using the logistic link:
\[
    \pi_i = \Pr(y_i = 1 \mid \hat\ba_i,\hat\bu_i)
          = \frac{1}{1 + \exp(-\eta_i)}.
\]
Let $\hat\bpi = (\pi_1,\dots,\pi_n)^\top$. The standard logistic negative
log-likelihood is
\[
    \ell_{\mathrm{logit}}(\beta_0,\balpha,\bbeta)
    = -\sum_{i=1}^n \left\{ 
        y_i \log \pi_i + (1-y_i)\log(1-\pi_i)
      \right\}.
\]

To encourage fairness, we penalize the dependence between the predictions and
the estimated sensitive attribute. Following \citet{zhao2022fairrf}, we use a
correlation-based measure
\[
    C(\hat\bA, \hat\bpi)
    = \bigl\|\hat\bA^\top \mathbf{C} \,\hat\bpi \bigr\|_2,
\]
where $\mathbf{C} = \mathbf{I}_n - \frac{1}{n}\mathbf{1}_n \mathbf{1}_n^\top$
is the centering matrix. Large values of $C(\hat\bA,\hat\bpi)$ indicate that
the predicted probabilities differ substantially across groups, while small
values indicate more similar predictions.

We therefore minimize the penalized objective
\begin{equation}
    \label{eq:logit-pen}
    \min_{\beta_0,\balpha,\bbeta}
        \ell_{\mathrm{logit}}(\beta_0,\balpha,\bbeta)
        + \lambda\, C(\hat\bA, \hat\bpi),
\end{equation}
where $\lambda \ge 0$ is a tuning parameter that controls the trade-off
between predictive accuracy and fairness. When $\lambda = 0$,
\eqref{eq:logit-pen} reduces to standard logistic regression with
$(\hat\ba_i,\hat\bu_i)$ as predictors. As $\lambda$ increases, the penalty
forces $C(\hat\bA,\hat\bpi)$ toward zero, yielding progressively fairer
classifiers at the cost of higher classification error.

Because $C(\hat\bA,\hat\bpi)$ is non-smooth at the origin, we employ a smooth
approximation
\[
    C_\delta(\hat\bA,\hat\bpi)
    = \sqrt{\bigl\|\hat\bA^\top\mathbf{C}\,\hat\bpi\bigr\|_2^2
            + \delta},
\]
for a small $\delta > 0$, and optimize the resulting smooth objective using
gradient-based methods (e.g., quasi-Newton or adaptive first-order
algorithms). In high-dimensional settings, we again couple this fairness
penalty with the SEMMS prior to obtain sparse, fairness-aware logistic
regression models.

In both regression and classification, varying the tuning parameters
$\varepsilon$ in \eqref{eq:R2-constraint} and $\lambda$ in
\eqref{eq:logit-pen} yields explicit accuracy–fairness trade-off curves, as
illustrated in the simulation and real-data studies in
Sections~\ref{sec:simulations} and~\ref{sec:realdata}.

\section{Theory}
\label{sec:theory}

In this section, we develop theoretical results for the two main components of our framework. First, we study identifiability of the mixture model used to infer the sensitive attribute and characterize the accuracy with which the latent classes can be recovered. Although identifiability of finite Gaussian mixture models is well established in the literature \citep{mclachlan2000finite}, our setting requires additional analysis because the sensitive attribute is unobserved and connected to the observed predictors through a generalized linear model. We consider both Gaussian and categorical mixtures, and show how classification performance depends on the separation between mixture components and the class prior probabilities.

Second, we examine how removing the linear dependence between the predictors and the latent sensitive attribute affects performance in a downstream regression task. Using $R^2$ as a measure of explanatory power, we derive closed-form expressions that make the trade-off between predictive accuracy and fairness explicit.

\subsection{Identifiability}
\label{subsec:identifiability}

While the identifiability of finite Gaussian mixture models is well understood, identifiability for categorical mixture models is more delicate and often poses substantial challenges.  
We begin with a simple motivating example and then provide a general counting argument for the product-multinomial case.

Consider first the simplest categorical mixture predictor
\[
x_i^{a} \;=\; a_i Z_{i1} + (1-a_i)Z_{i2}, \qquad
a_i \sim \mathrm{Ber}(p),\quad
Z_{i1} \sim \mathrm{Ber}(\theta_1),\quad
Z_{i2} \sim \mathrm{Ber}(\theta_2).
\]
This model involves $n_{\text{mix}} = 3$ parameters $(p,\theta_1,\theta_2)$, but the observable $x_i^{a}$ takes only two possible values: $0$ and $1$. Hence, the distribution of $x_i^{a}$ yields only one free probability (since probabilities must sum to one). In other words, the model is over-parameterized relative to what can be learned from the data, and the parameters are not identifiable.

In the general case, suppose we have a $K$-component categorical mixture with $D$ independent categorical variables, where the $d$th variable has $M_d$ levels. The mixture model then contains
\[
n_{\text{mix}} \;=\; (K-1) + K\sum_{d=1}^D (M_d - 1)
\]
free parameters: $K-1$ mixture proportions and, for each of the $K$ classes, $\sum_{d=1}^D (M_d - 1)$ free parameters in the conditional multinomial probabilities.  
On the other hand, the joint distribution of the $D$ categorical variables has $\prod_{d=1}^D M_d$ possible combinations and therefore requires
\[
n_{\text{joint}} \;=\; \prod_{d=1}^D M_d - 1
\]
parameters to be specified. A necessary (but not sufficient) dimension condition for generic identifiability is that the number of free mixture parameters not exceed the degrees of freedom in the observable joint distribution, that is,
\[
\prod_{d=1}^D M_d - 1 \;\ge\; (K-1) + K\sum_{d=1}^D (M_d - 1).
\]
For example, if $K = 2$ and $D = 2$, and both variables have the same number of levels $M_1 = M_2 = M$, this inequality requires $M \ge 4$; alternatively, if $K = 2$ and $D = 3$, it suffices to have $M_d = 2$ (i.e., three binary variables).  
This counting argument can rule out identifiability, but satisfying it does not by itself prove identifiability. For latent class and related product-multinomial models, generic identifiability typically requires additional algebraic or rank conditions; see, for example, \citet{allman2009identifiability} for general sufficient conditions based on Kruskal-type arguments. Thus, in our setting the inequality above should be interpreted only as a necessary screening rule: when it fails, identifiability is impossible, whereas when it holds, one must still verify additional structural conditions.

\subsection{Accuracy in determining the sensitive variable: Gaussian mixture}
\label{subsec:gauss-accuracy}

We now turn to the accuracy with which the latent sensitive attribute can be recovered when the predictors follow a Gaussian mixture.  
Throughout this subsection, we assume that the mixture parameters are known and focus on the probability that the posterior mode classifier correctly recovers the true component label.

\begin{theorem}[Multivariate Gaussian mixture]
\label{thm:gaussian-multi}
Suppose
\[
[\bx_i^{a} \mid A_i = k] \sim N(\bmu_k,\bSigma_e),\qquad k=1,\ldots,K,
\]
where \(\bSigma_e\) is positive definite, and define
\[
\delta_{k\ell}=\{(\bmu_k-\bmu_\ell)^\top\bSigma_e^{-1}(\bmu_k-\bmu_\ell)\}^{1/2},
\qquad
\delta_{\min}=\min_{k\ne\ell}\delta_{k\ell}.
\]
Let
\[
\mathcal D_k=\left\{x:
\log p_k-\tfrac12(x-\bmu_k)^\top\bSigma_e^{-1}(x-\bmu_k)
\ge
\log p_\ell-\tfrac12(x-\bmu_\ell)^\top\bSigma_e^{-1}(x-\bmu_\ell)
\ \text{for every }\ell\ne k
\right\}
\]
be the posterior-mode decision region for class \(k\), with a fixed rule for ties. Then
\[
\Pr(\hat A_i=A_i)=\sum_{k=1}^K p_k\int_{\mathcal D_k}
\phi_{p_a}(x;\bmu_k,\bSigma_e)\,dx.
\]
Moreover, if all class probabilities are positive and \(\delta_{\min}\to\infty\), then
\(\Pr(\hat A_i=A_i)\to1\).
\end{theorem}

In the special case where $\bx_i^{a}$ is univariate and
\[
[x_i^{a}\mid A_i=j] \sim N(\mu_j,\sigma_e^2), \qquad j=1,\ldots,K,
\]
it is intuitive that the larger the separation between component means, the easier it becomes to identify the true latent class. The next result provides an explicit separation condition that guarantees a desired level of classification accuracy.

\begin{theorem}[Univariate Gaussian mixture]
\label{thm:gaussian-uni}
Suppose \(X\mid A=k\sim N(\mu_k,\sigma_e^2)\), let
\(\delta_{k\ell}=|\mu_k-\mu_\ell|/\sigma_e\), and let
\(\delta_{\min}=\min_{k\ne\ell}\delta_{k\ell}\). Define the decision interval (possibly empty or unbounded)
\[
\mathcal I_k=\left\{x:
\log p_k-\frac{(x-\mu_k)^2}{2\sigma_e^2}
\ge
\log p_\ell-\frac{(x-\mu_\ell)^2}{2\sigma_e^2}
\ \text{for every }\ell\ne k
\right\}.
\]
Then the exact posterior-mode accuracy is
\[
\Pr(\hat A_i=A_i)=\sum_{k=1}^K p_k
\int_{\mathcal I_k}\frac{1}{\sigma_e}\phi\!\left(\frac{x-\mu_k}{\sigma_e}\right)dx.
\]
For a target error level \(\alpha\in(0,1)\), put
\[
z_{\alpha,K}=\Phi^{-1}\!\left(1-\frac{\alpha}{K-1}\right),
\qquad r_p=\frac{p_{\max}}{p_{\min}}.
\]
A sufficient separation condition for \(\Pr(\hat A_i=A_i)\ge1-\alpha\) is
\begin{equation}
\delta_{\min}\ge z_{\alpha,K}+\sqrt{z_{\alpha,K}^2+2\log r_p}.
\label{eq:delta-condition}
\end{equation}
\end{theorem}

Thus, for Gaussian mixtures, both the separation between component means and the imbalance in class probabilities jointly determine how accurately the latent sensitive attribute can be recovered from $\bx_i^{a}$.

\subsubsection{Accuracy in determining the sensitive variable: categorical mixture}
\label{subsubsec:cat-accuracy}

We next consider a product-multinomial mixture model for which the predictors are categorical. Again, we assume that all mixture parameters are known and quantify the probability that the posterior mode classifier correctly identifies the true latent class.

Let $D$ denote the number of categorical predictors, and let the $d$th predictor have $M_d$ categories. For component $k$, let
\[
\btheta_{k,d} = (\theta_{k,d,1},\ldots,\theta_{k,d,M_d})^\top
\]
denote the probability vector governing the distribution of the $d$th predictor, with
$\theta_{k,d,j} = \Pr(X_d = j \mid A = k)$.  

\begin{theorem}[General categorical mixture]
\label{thm:categorical-general}
Assume that all mixture parameters are known and that the posterior class scores are distinct for every category pattern having positive probability. The probability of correctly estimating the sensitive attribute from a product-multinomial mixture under the posterior mode rule is
\[
\Pr\bigl(\arg\max_{l} \hat a_{il} = A_i\bigr)
=
\sum_{k=1}^K
p_k
\sum_{\text{all combinations } j^*}
\left[
\prod_{l \neq k}
\mathbbm{1}\!\left\{
\prod_{d=1}^D
\frac{\theta_{k,d,j^*_d}}{\theta_{l,d,j^*_d}}
>
\frac{p_l}{p_k}
\right\}
\;
\prod_{d=1}^D \theta_{k,d,j^*_d}
\right],
\]
where the inner sum runs over all possible category combinations
$j^* = (j^*_1,\ldots,j^*_D)$ with $j^*_d \in \{1,\ldots,M_d\}$, and
$\mathbbm{1}\{\cdot\}$ denotes the indicator function.
The proof is given in Appendix~\ref{app:proof-cat-general}.
\end{theorem}

In the simplest case, 
\[
X = a Z_1 + (1-a)Z_2,\qquad a \sim \mathrm{Ber}(p),\quad Z_1 \sim \mathrm{Ber}(\theta_1),\quad Z_2 \sim \mathrm{Ber}(\theta_2),
\]
the expression becomes more explicit.

\begin{theorem}[Binary categorical mixture]
\label{thm:categorical-binary}
In the binary case described above,
\[
\Pr\bigl(\arg\max_{l} \hat a_{il} = A_i\bigr)
=
(m - n)\big(p\theta_1 - (1-p)\theta_2\big) + n(2p-1) + (1-p),
\]
where
\[
m = \mathbbm{1}\!\left(\frac{\theta_1}{\theta_2} \ge \frac{1-p}{p}\right),
\qquad
n = \mathbbm{1}\!\left(\frac{1-\theta_1}{1-\theta_2} \ge \frac{1-p}{p}\right),
\]
and $m,n \in \{0,1\}$ are indicator variables (not the sample size).
The proof is given in Appendix~\ref{app:proof-cat-binary}.
\end{theorem}

From this expression, we see that as the component probabilities become more extreme, for example, $|\theta_1 - \theta_2| \to 1$ or $p \to 0$ (or $p \to 1$), we approach a best-case scenario with near-perfect classification. Otherwise, when the two components are not well separated and the prior probabilities are balanced, performance can be no better than random guessing.  
Combined with the identifiability discussion in Section~\ref{subsec:identifiability}, this highlights that the accurate recovery of the sensitive attribute from categorical predictors requires both sufficient component separation and an identifiable mixture model. This logic is summarized schematically in Figure~\ref{fig:flowchart}.

\begin{figure}[h]
    \centering
    \includegraphics[width=0.88\textwidth]{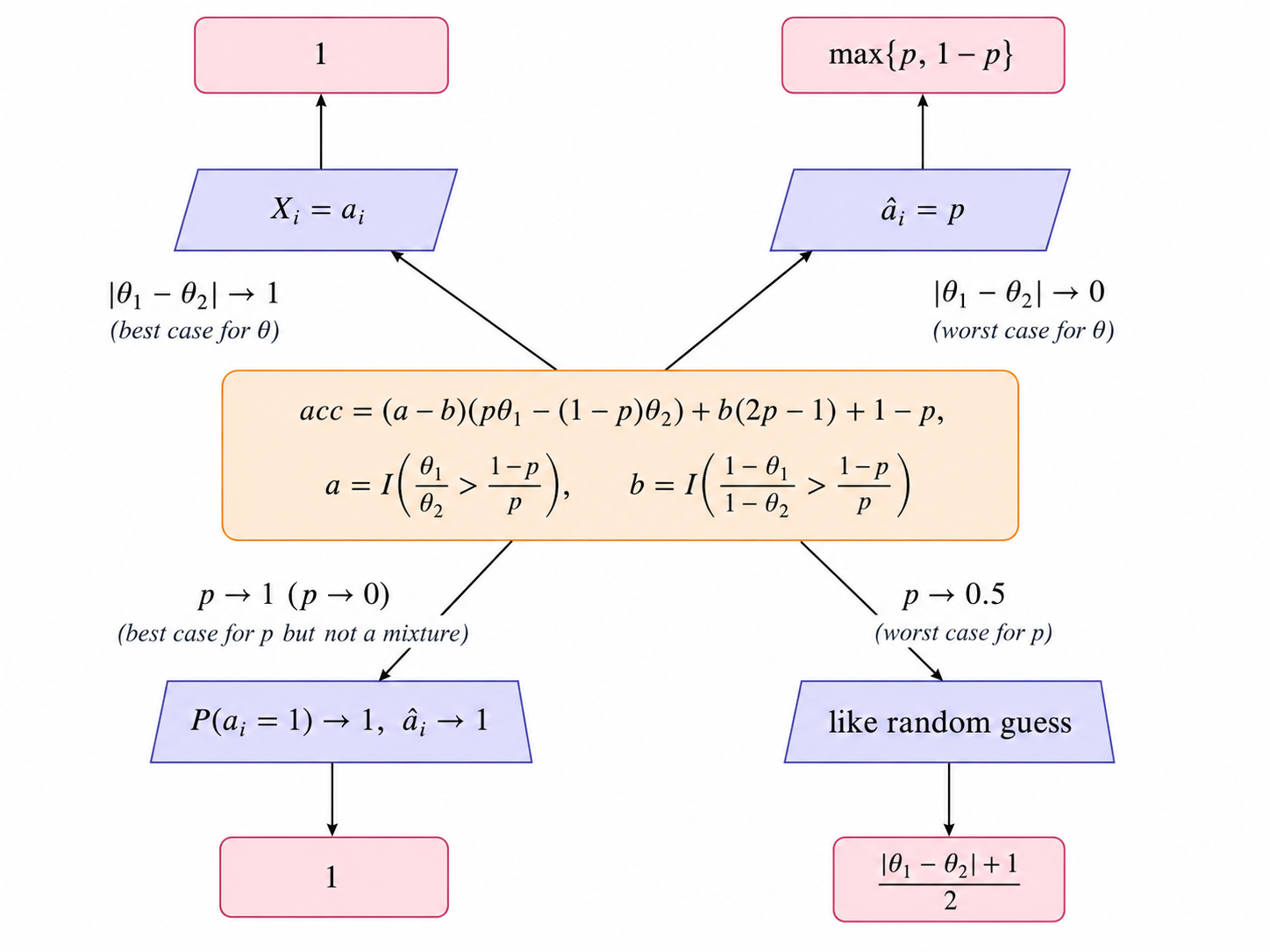}
    \caption{Posterior-mode classification accuracy in the binary categorical mixture and its limiting regimes. The central block gives the exact accuracy formula, the blue blocks summarize intermediate cases, and the pink blocks show the corresponding limiting accuracies.}
    \label{fig:flowchart}
\end{figure}

\subsection{Loss of predictive power and the accuracy--fairness trade-off}
\label{subsec:r2-tradeoff}

We now investigate how enforcing fairness by removing the linear association between the predictors and the sensitive attribute impacts predictive performance.  
Following the procedure described in Section~\ref{subsec:residual}, we regress the predictors on the estimated sensitive attribute and work with the residuals; here, we analyze the resulting loss in $R^2$.

\begin{model}
\label{model:3.1}
For $i = 1,\ldots,n$,
\[
y_i = \beta_0 + \bbeta_a^\top \bx_i^{a} + \bbeta_z^\top \bx_i^{z} + \varepsilon_i,
\]
where $\bbeta_a \in \mathbb{R}^{p_a}$, $\bbeta_z \in \mathbb{R}^{p_z}$,
$\bx_i^{a} \in \mathbb{R}^{p_a}$ denotes the block of predictors related to the sensitive attribute, and $\bx_i^{z} \in \mathbb{R}^{p_z}$ denotes the remaining predictors.  

Assume
\[
[\bx_i^{a} \mid A_i = j] \sim N(\bmu_j,\bSigma_e), \quad j = 1,\ldots,K,
\]
which can be written compactly as
\[
\bx_i^{a} = \bmu^\top \ba_i + \be_i,
\]
with $\bmu \in \mathbb{R}^{K \times p_a}$, $\ba_i\in \{0,1\}^K$ the $K$-dimensional one-hot representation of $A_i$, and
$\be_i \sim N(\mathbf{0},\bSigma_e)$, $\bSigma_e \in \mathbb{R}^{p_a \times p_a}$.  
Assume further that, conditional on $A_i=k$,
\[
\begin{pmatrix}\bx_i^{a}\\[2pt]\bx_i^{z}\end{pmatrix}
\sim
N\!\left(
\begin{pmatrix}\bmu_k\\[2pt]\mathbf{0}\end{pmatrix},
\begin{pmatrix}
\bSigma_e & \bSigma_{12}\\
\bSigma_{12}^\top & \bSigma_{xz}
\end{pmatrix}
\right).
\]
Thus $E(\bx_i^{z}\mid A_i)=\mathbf{0}$, $\Var(\bx_i^{z})=\bSigma_{xz}$, and
\[
\Var(\bx_i^{a})=\bmu^\top\bSigma_A\bmu+\bSigma_e,
\qquad
\Cov(\bx_i^{a},\bx_i^{z})=\bSigma_{12},
\]
where $\bSigma_A=\Var(\ba_i)$. This formulation distinguishes the conditional covariance $\bSigma_e$ from the between-class variation $\bmu^\top\bSigma_A\bmu$ in the marginal covariance of $\bx_i^{a}$.
\end{model}

Let $R_X^2$ denote the $R^2$ when regressing $\by$ on the full predictor vector $\bx = (\bx^{a}, \bx^{z})$, and let $R_U^2$ denote the $R^2$ after the fairness adjustment---i.e., after regressing $\bx$ on $\ba$ and using the residuals $\bu$ as predictors. The following result gives asymptotic expressions for both quantities.

\begin{theorem}[Linear regression case]
\label{thm:R2-general}
Under Model~\ref{model:3.1},
\[
R_U^2
=
\frac{
\bbeta_a^\top \bSigma_e \bbeta_a
+
\bbeta_z^\top \bSigma_{xz} \bbeta_z
+
2\bbeta_a^\top \bSigma_{12} \bbeta_z
}{
\bbeta_a^\top \bmu^\top \bSigma_A \bmu \bbeta_a
+
\bbeta_a^\top \bSigma_e \bbeta_a
+
\bbeta_z^\top \bSigma_{xz} \bbeta_z
+
2\bbeta_a^\top \bSigma_{12} \bbeta_z
+
\sigma_\varepsilon^2
}
+
O_p(n^{-1/2}),
\]
and
\[
R_X^2
=
\frac{
\bbeta_a^\top \bmu^\top \bSigma_A \bmu \bbeta_a
+
\bbeta_a^\top \bSigma_e \bbeta_a
+
\bbeta_z^\top \bSigma_{xz} \bbeta_z
+
2\bbeta_a^\top \bSigma_{12} \bbeta_z
}{
\bbeta_a^\top \bmu^\top \bSigma_A \bmu \bbeta_a
+
\bbeta_a^\top \bSigma_e \bbeta_a
+
\bbeta_z^\top \bSigma_{xz} \bbeta_z
+
2\bbeta_a^\top \bSigma_{12} \bbeta_z
+
\sigma_\varepsilon^2
}
+
O_p(n^{-1/2}).
\]
The proof is given in Appendix~\ref{app:proof-R2-general}.
\end{theorem}

Theorem~\ref{thm:R2-general} is an oracle calculation in which the true sensitive attribute is used for residualization. When $A$ is replaced by $\hat A$, additional error enters through latent-class estimation; quantifying that contribution requires assumptions on mixture separation and parameter-estimation rates and is left for future work.

In the basic case where $x_i^{a}$ is univariate and $A_i$ is binary, these expressions simplify substantially.

\begin{theorem}[Univariate Gaussian case]
\label{thm:R2-univariate}
Consider the simplest version of Model~\ref{model:3.1} in which
\[
[x_i \mid A_i = 1] \sim N(\mu,\sigma_e^2), \qquad
[x_i \mid A_i = 0] \sim N(0,\sigma_e^2), \qquad
A_i \sim \mathrm{Ber}(p),
\]
and $y_i = \beta_0 + \beta_1 x_i + \varepsilon_i$. Then
\[
R_X^2
=
\frac{\beta_1^2 \big( p(1-p)\mu^2 + \sigma_e^2 \big)}
     {\beta_1^2 p(1-p)\mu^2 + \beta_1^2 \sigma_e^2 + \sigma_\varepsilon^2}
     + O_p(n^{-1/2}),
\]
which is the $R^2$ obtained by regressing $y$ on $x$, and
\[
R_U^2
=
\frac{\beta_1^2 \sigma_e^2}
     {\beta_1^2 p(1-p)\mu^2 + \beta_1^2 \sigma_e^2 + \sigma_\varepsilon^2}
     + O_p(n^{-1/2}),
\]
which is the $R^2$ obtained after the fairness adjustment, i.e., after regressing $y$ on the residuals $u$ from regressing $x$ on $A$.  
The proof is given in Appendix~\ref{app:proof-R2-univariate}.
\end{theorem}

From these expressions, we observe that the term $p(1-p)\mu^2$ appears in both the numerator and denominator of $R_X^2$, but is absent from $R_U^2$. As $\mu$ increases, corresponding to greater separation between the distributions of $x$ for different values of $A$, this term grows, leading to the following behavior:
\begin{itemize}
    \item As $\mu \to \infty$, the explanatory power of $x$ increases and $R_X^2 \to 1$, indicating that $x$ essentially fully captures the variation in $y$.
    \item In contrast, $R_U^2$ decreases toward zero because, after removing the effect of $A$, the remaining residual variation in $x$ carries little information about $y$.
    \item When $\mu \to 0$, implying no separation between the groups, $R_X^2$ and $R_U^2$ become approximately equal, since $x$ no longer meaningfully encodes the sensitive attribute.
\end{itemize}

This analysis highlights a fundamental accuracy–fairness trade-off: increasing $\mu$ improves the predictive accuracy of models that exploit the sensitive structure in $x$, but once fairness interventions remove the component associated with $A$, the explanatory power of the remaining predictors diminishes, leading to smaller $R_U^2$. This is consistent with the broader literature showing that fairness constraints may induce nontrivial trade-offs with predictive accuracy or calibration \citep{menon2018cost,kleinberg2017tradeoffs,pleiss2017calibration}.
\section{Simulations}
\label{sec:simulations}

\subsection{Sensitive attribute estimation accuracy (Gaussian)}
\label{subsec:sim-gaussian-acc}

We first study the accuracy of estimating the latent sensitive attribute in a Gaussian mixture setting and compare the empirical performance of the EM estimator with the theoretical guarantee in Theorem~\ref{thm:gaussian-uni}.

Let $n = 1000$, $\alpha = 0.05$, and $\sigma = 1$. We consider a univariate Gaussian mixture with $K = 3$ components and component means
\[
\mu_1 = 0, 
\qquad 
\mu_2 = \mu_{\min}, 
\qquad 
\mu_3 = 2\mu_{\min} + 1,
\]
and common variance $\sigma^2$. The class prior probabilities are
\[
\bp = (p_1,p_2,p_3)^\top = (0.2,0.3,0.5)^\top.
\]
For each $i = 1,\dots,n$, we draw $A_i \sim \mathrm{Multi}(1,\bp)$ and then
\[
[x_i^{a} \mid A_i = j] \sim N(\mu_j,\sigma^2),\qquad j=1,2,3.
\]
Denote by $\hat A_i$ the posterior mode classifier constructed from the fitted mixture model and by $\hat a_i$ its corresponding one-hot encoding. According to Theorem~\ref{thm:gaussian-uni} and the separation condition in Eq.~\eqref{eq:delta-condition}, the minimal separation $\mu_{\min}^*$ required to achieve a target error level $\alpha = 0.05$ can be computed from the class priors and noise level, yielding $\mu_{\min}^* \approx 4.33$ in this setting. We vary
\(
\mu_{\min} \in \{0.5,\,1,\,2,\,3,\,\mu_{\min}^*,\,2\mu_{\min}^*\}
\)
and, over 10 replications for each value of $\mu_{\min}$, compute the fraction of correctly estimated sensitive attributes \(
\frac{1}{n}\sum_{i=1}^n \mathbf{1}(\hat A_i = A_i). \)

\begin{figure}[h]
    \centering
    \includegraphics[width=\textwidth]{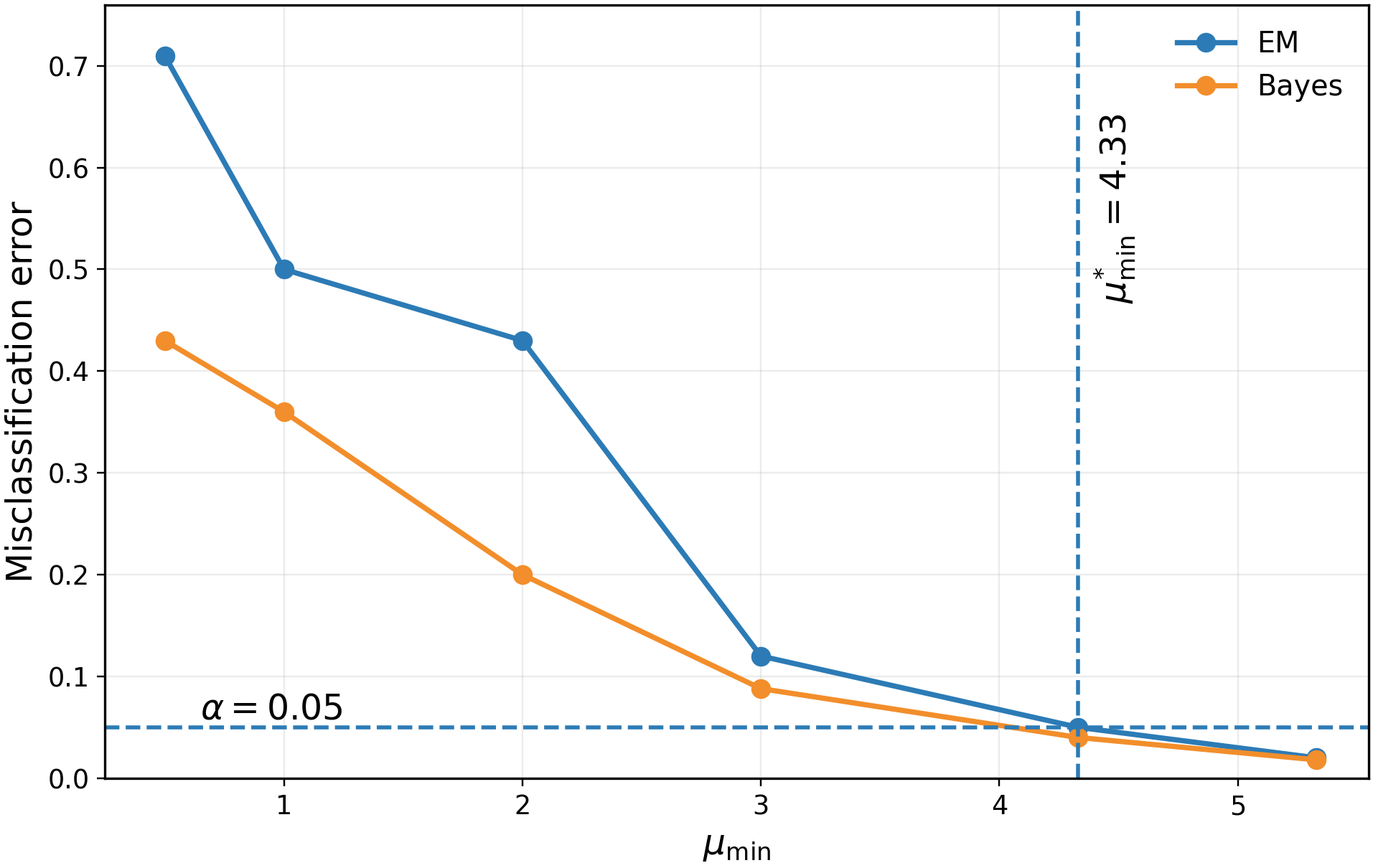}
    \caption{Gaussian mixture model: misclassification error for $\hat A$ versus $\mu_{\min}$.}
    \label{fig:error-rate-mu-min}
\end{figure}

As shown in Figure~\ref{fig:error-rate-mu-min}, the misclassification error is essentially below $0.05$ once $\mu_{\min} \ge \mu_{\min}^* \approx 4.33$, in line with the theoretical prediction of Theorem~\ref{thm:gaussian-uni}.

\subsection{Sensitive attribute estimation accuracy (categorical)}
\label{subsec:sim-cat-acc}

We next consider a product-multinomial mixture with $K=2$ classes, $D=3$ categorical predictors, and $M_d = 2$ categories for each predictor. For each $i = 1,\dots,n$,
\[
A_i \sim \mathrm{Ber}(p),\qquad n = 10000,
\]
and
\begin{align*}
X_{1i} &= A_i Z_{1,1i} + (1-A_i) Z_{2,1i},\\
X_{2i} &= A_i Z_{1,2i} + (1-A_i) Z_{2,2i},\\
X_{3i} &= A_i Z_{1,3i} + (1-A_i) Z_{2,3i},
\end{align*}
where, for $d=1,2,3$,
\[
Z_{1,di} \sim \mathrm{Ber}(\theta_{1,d}),\qquad
Z_{2,di} \sim \mathrm{Ber}(\theta_{2,d}),
\]
independently over $i$ and $d$.

We examine two aspects:

1. \textbf{Effect of class prior $p$}.  
   First, we let the component probabilities be separable, in the sense that
   \[
   \min_{d} |\theta_{1,d} - \theta_{2,d}| \neq 0,
   \]
   and study how the accuracy of estimating $\hat A$ varies with the class prior $p$.  
   Second, we set the component probabilities to be identical,
   \[
   |\theta_{1,d} - \theta_{2,d}| = 0 \quad \text{for all } d,
   \]
   and again examine the dependence of the classification accuracy on $p$.

   Figure~\ref{acc_p} compares the empirical classification accuracy from repeated simulations with the theoretical accuracy given by Theorem~\ref{thm:categorical-general}. 
For each value of $p$, each dot represents the accuracy from one independent replicate (Monte Carlo run), while the red curve shows the corresponding theoretical accuracy.
When the $\theta$'s are not identifiable (right panel), the accuracy is essentially linear in $|p-0.5|$, reflecting the fact that the classifier effectively falls back on the prior.
When the $\theta$'s are identifiable (left panel), the accuracy is no longer linear in $|p-0.5|$: it is minimized at $p=0.5$ and increases toward $1$ as $p\to 0$ or $p\to 1$.

\begin{figure}[h!]
    \centering
    \begin{subfigure}[b]{\textwidth}
        \centering
        \includegraphics[width=\textwidth]{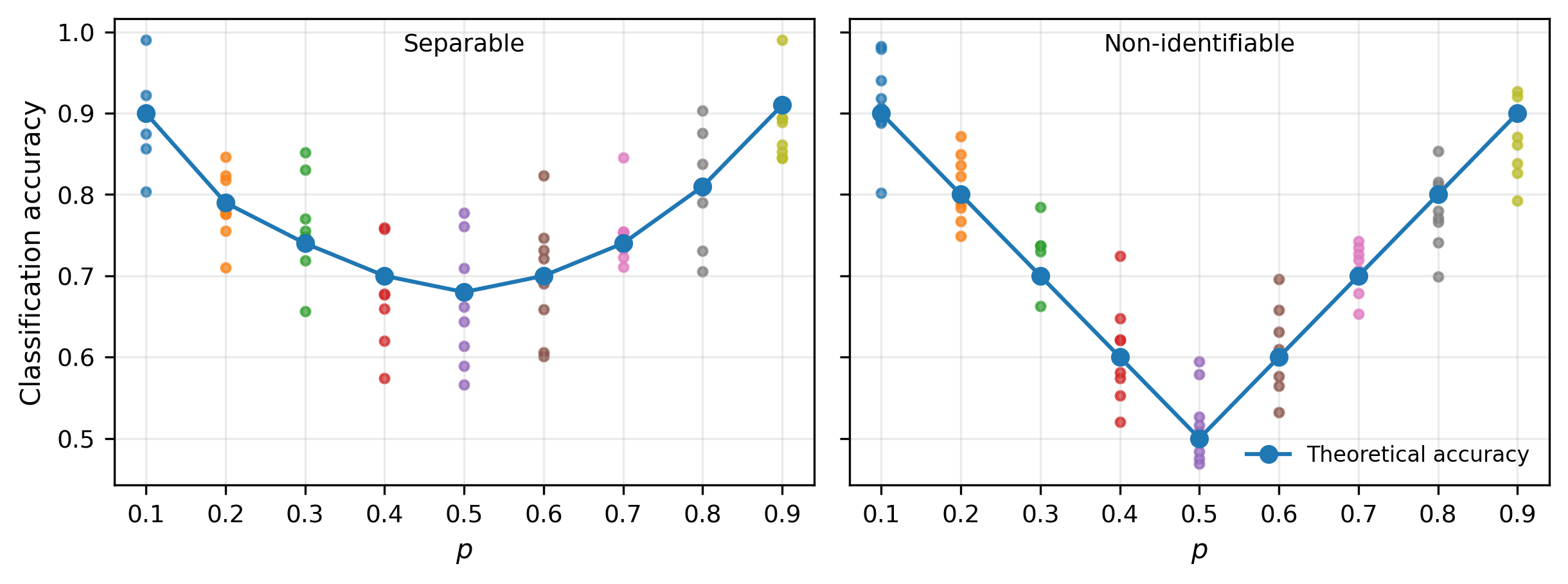}
    \end{subfigure}
    \caption{Categorical mixture model: accuracy of $\hat A$ versus $p$.  
    Left: separable case with $\theta_{1,1}=0.2,\theta_{1,2}=0.6,\theta_{1,3}=0.2$, 
    $\theta_{2,1}=0.4,\theta_{2,2}=0.3,\theta_{2,3}=0.5$.  
    Right: non-identifiable case with $\theta_{1,1}=0.5,\theta_{1,2}=0.6,\theta_{1,3}=0.4$, 
    $\theta_{2,1}=0.5,\theta_{2,2}=0.6,\theta_{2,3}=0.4$.  
    The red line shows the theoretical accuracy from Theorem~\ref{thm:categorical-general}.}
    \label{acc_p}
\end{figure}

2.  \textbf{Effect of component separation}.  
   We then investigate how the accuracy of estimating $\hat A$ varies with the separation between component probabilities when $p$ is fixed. Figure~\ref{acc_theta} shows the accuracy as a function of $|\theta_{2,1}-\theta_{1,1}|$ with $\theta_{1,1}=0.2$, $\theta_{1,2}=0.6$, $\theta_{2,2}=0.3$, $\theta_{1,3}=0.2$, and $\theta_{2,3}=0.5$. As expected, the accuracy increases monotonically with the difference between the component probabilities.

\begin{figure}[h]
    \centering
    \includegraphics[width=10 cm]{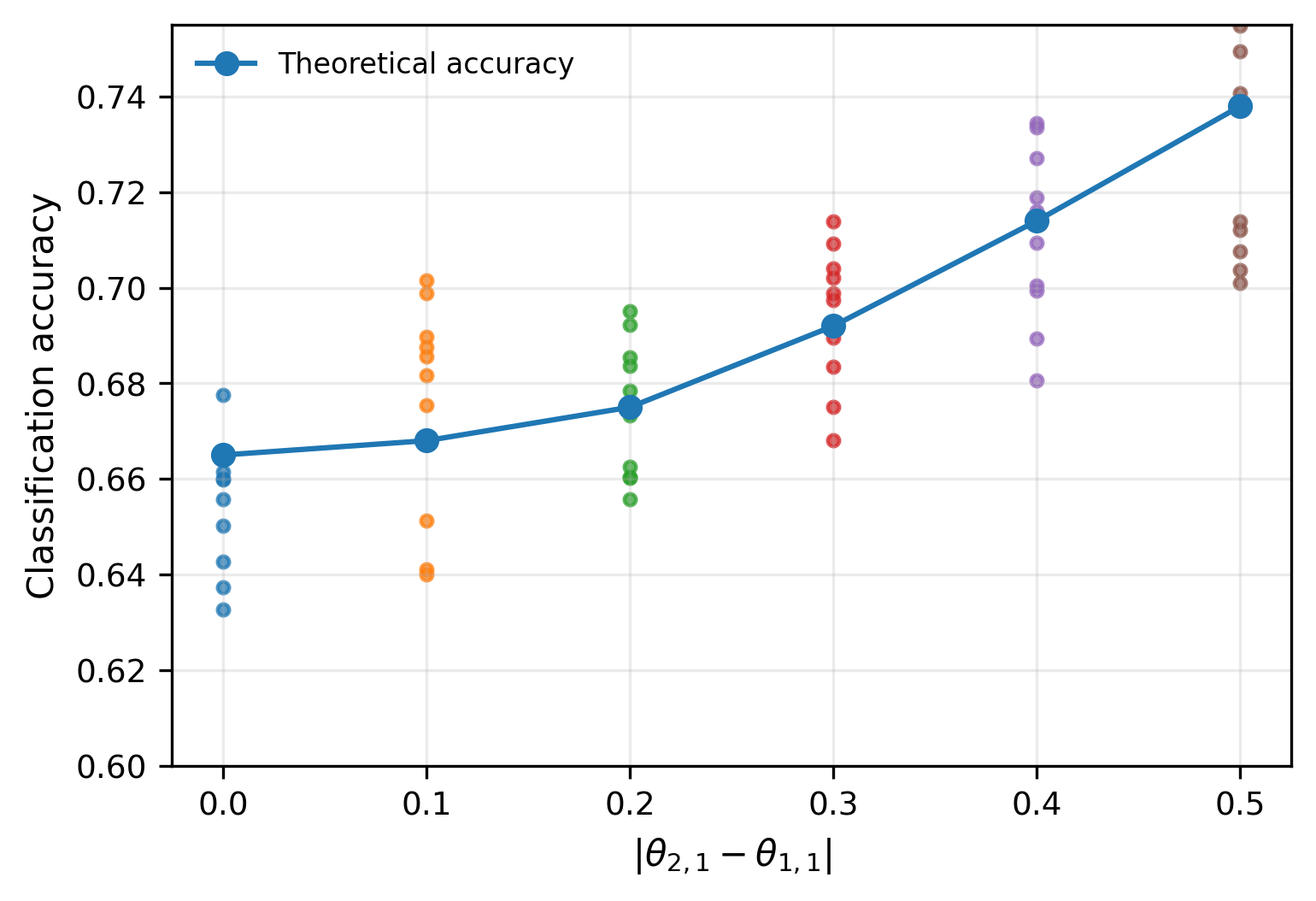}
    \caption{Categorical mixture model: accuracy of $\hat A$ versus $|\theta_{2,1}-\theta_{1,1}|$,  
    with $\theta_{1,1}=0.2,\theta_{1,2}=0.6,\theta_{2,2}=0.3,\theta_{1,3}=0.2,\theta_{2,3}=0.5$.}
    \label{acc_theta}
\end{figure}

\subsection{Adjusted $R^2$ (Gaussian)}
\label{subsec:sim-R2}

We now illustrate the loss of predictive power under the fairness adjustment described in Section~\ref{subsec:r2-tradeoff} and Theorem~\ref{thm:R2-univariate}.

Let $x_i^{a}$ be univariate and generated as
\[
x_i^{a} = \mu A_i + e_i,
\]
where $A_i \sim \mathrm{Ber}(0.7)$, $e_i \sim N(0,\sigma_e^2)$ with $\sigma_e^2 = 2$. Let $\bx_i^{z} \in \mathbb{R}^2$ with independent coordinates drawn from $\mathrm{Unif}(0,1)$, and define the response
\[
y_i = \beta_0 + \beta_1 x_i^{a} + \bbeta_z^\top \bx_i^{z} + \varepsilon_i,
\]
with $\beta_0 = \beta_1 = 1$, $\bbeta_z = (1,1)^\top$, and $\varepsilon_i \sim N(0,\sigma_\varepsilon^2)$ with $\sigma_\varepsilon^2 = 0.5^2$.

For each choice of $\mu$, we compute the empirical $R_U^2$ obtained by regressing $y$ on the residuals from regressing $(x^{a},\bx^{z})$ on $A$. Figure~\ref{fig:R2-vs-n} shows $R_U^2$ as a function of $n$ for 500 Monte Carlo replications. For example, when $\mu = 6$, Theorem~\ref{thm:R2-univariate} predicts a limiting value of $R_U^2 \approx 0.13$, while for $\mu = 10$ the predicted limit is approximately $0.052$. In both cases, the empirical $R_U^2$ converges neatly to its theoretical limit.

\begin{figure}[h]
    \centering
    \includegraphics[width=\textwidth]{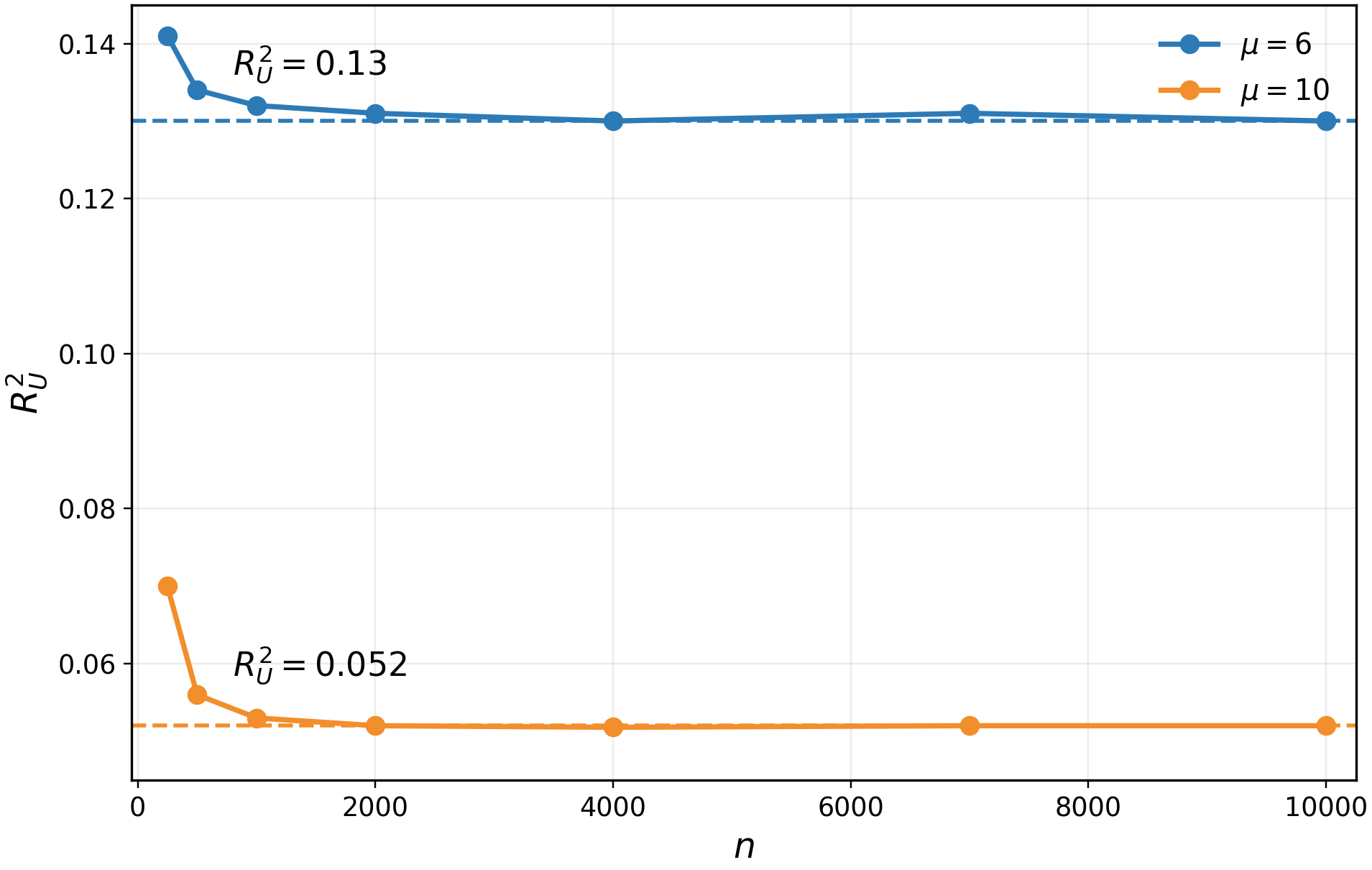}
    \caption{Gaussian mixture regression: $R_U^2$ versus $n$ for different values of $\mu$, averaged over 500 replications.}
    \label{fig:R2-vs-n}
\end{figure}

\subsection{Least-squares regression (Gaussian)}
\label{subsec:sim-ls}

We next consider a higher-dimensional Gaussian mixture regression to study the interaction between fairness constraints and variable selection.

Let $n=100$, and define
\[
\bmu_1 = (0,0)^\top,\qquad
\bmu_2 = (\mu,0)^\top,\qquad
\bmu_3 = (k\mu,0)^\top,
\]
for values of $\mu$ and $k$ specified below. Let $A_i \sim \mathrm{Ber}(p)$ with $p=0.7$. We generate four primary predictors via two bivariate mixtures:
\begin{align*}
(Z_{1i}, Z_{2i})^\top &\sim 
\begin{cases}
N(\bmu_2,\bSigma), & A_i = 1,\\[2pt]
N(\bmu_1,\bSigma), & A_i = 0,
\end{cases}
\\[4pt]
(Z_{3i}, Z_{4i})^\top &\sim 
\begin{cases}
N(\bmu_3,\bSigma), & A_i = 1,\\[2pt]
N(\bmu_1,\bSigma), & A_i = 0,
\end{cases}
\end{align*}
with
\[
\bSigma =
\begin{pmatrix}
1 & 0.5\\
0.5 & 1
\end{pmatrix}.
\]
In addition, we generate $Z_{5i},\dots,Z_{104i} \sim \mathrm{Unif}(-2,2)$ independently and define the response
\[
y_i = Z_{1i} + Z_{2i} + Z_{5i} + Z_{6i} + \varepsilon_i,\qquad
\varepsilon_i \sim N(0,0.5^2).
\]

We compare, over 10 replications, the ratio $\mathrm{SSE}/\mathrm{SST}$ as a function of the fairness tuning parameter $\varepsilon$ in three scenarios:
\begin{itemize}
    \item different separations $\mu \in \{2,6\}$;
    \item different scaling factors $k \in \{1/2, 2\}$;
    \item the ``correct'' model versus a model using SEMMS for variable selection.
\end{itemize}

In this setup, when $\mu=2$ and $k=2$, SEMMS consistently selects $(Z_1,Z_2,Z_5,Z_6)$, and the EM-based screening step correctly identifies $Z_1$ as the relevant mixture-based variable in 7 out of 10 runs.  
When $\mu=6$ and $k=2$, SEMMS selects $Z_3$ (which is absent from the true model) in about half of the runs, and the EM step also tends to choose $Z_3$ because of its larger mean.  
When $\mu=6$ and $k=1/2$, SEMMS again selects $(Z_1,Z_2,Z_5,Z_6)$, and the EM step always chooses $Z_1$ as the relevant variable, since its mean dominates that of $Z_3$.

\begin{figure}[h]
\centering
\includegraphics[width=15 cm]{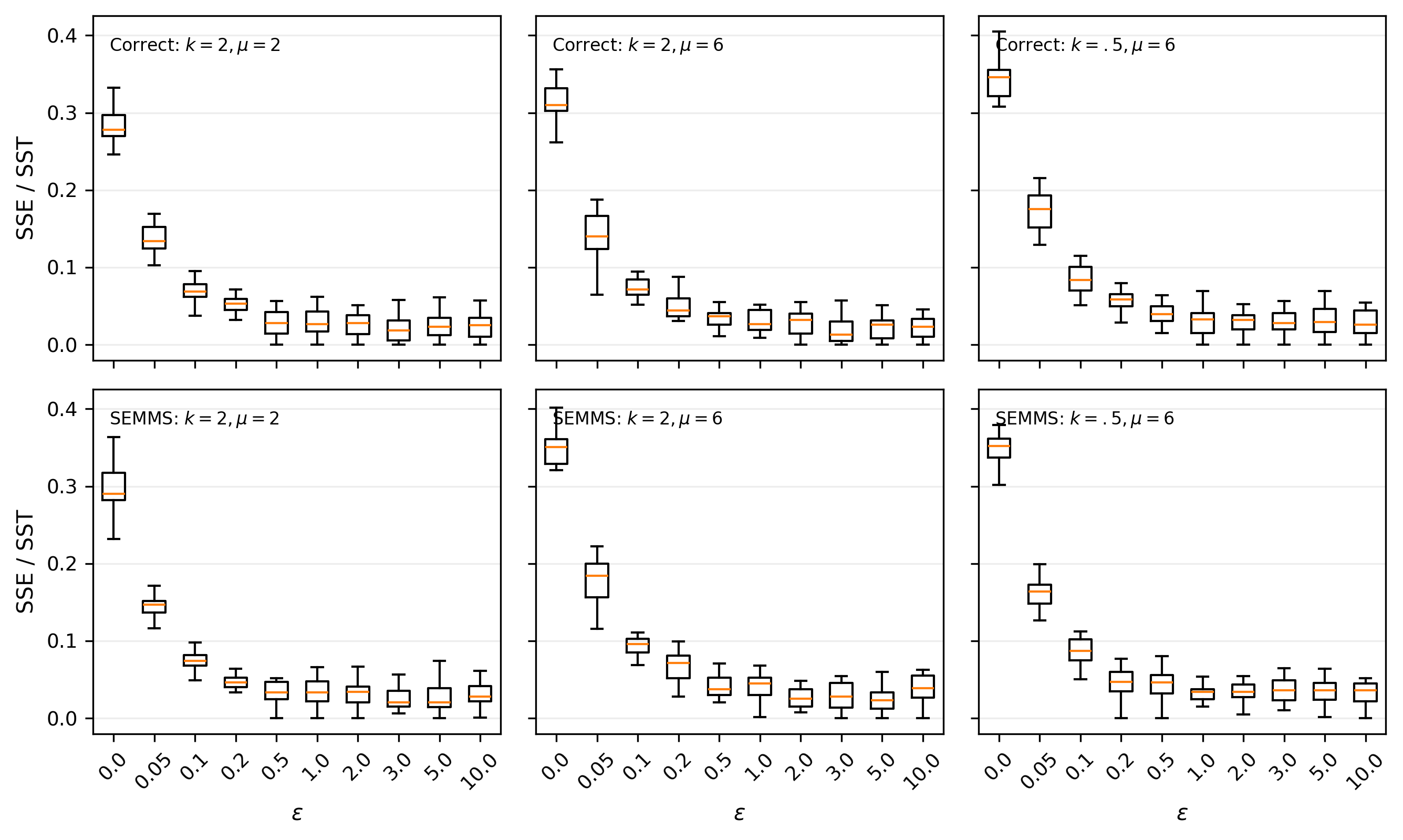}
\caption{Gaussian mixture regression (Model 4): $\mathrm{SSE}/\mathrm{SST}$ versus fairness parameter $\varepsilon$ for different combinations of $(\mu,k)$ and with/without SEMMS variable selection.}
\label{Figure3}
\end{figure}

When $k=2$ and $\mu=2$, the error curves for the correct model and the SEMMS-selected model are quite similar across values of $\varepsilon$. In contrast, when $k=2$ and $\mu=6$, incorporating variable selection leads to a higher error rate due to the spurious inclusion of $Z_3$, which is not part of the true model but has a large mean. For $k=1/2$, the behavior of the correct and SEMMS models is again similar, reflecting the fact that SEMMS correctly avoids selecting $Z_3$ in this regime.

\subsection{Binary classification (Gaussian)}
\label{subsec:sim-bin-gaussian}

We now turn to the binary classification setting under a Gaussian mixture, using the fairness-aware logistic regression formulation in Eq.~\eqref{eq:logit-pen}.

The data-generating mechanism for $(Z_{1},Z_{2},Z_{3},Z_{4})$ and the sensitive attribute $A$ is the same as in Section~\ref{subsec:sim-ls}, except that we now generate $Z_{5},\dots,Z_{100} \sim \mathrm{Unif}(-2,2)$ and define the linear predictor
\[
\eta_i = 1 + Z_{1i} + Z_{2i} + Z_{5i} + Z_{6i} + Z_{7i},
\]
followed by
\[
\pi_i = \frac{1}{1 + \exp(-\eta_i)},\qquad
Y_i \sim \mathrm{Ber}(\pi_i).
\]

We minimize the penalized objective in Eq.~\eqref{eq:logit-pen} using gradient descent with a numerical approximation to the gradient of the fairness penalty. Over 10 replications, we compare:
\begin{itemize}
    \item $\mu \in \{2,6\}$;
    \item $\lambda \in \{0, 0.1, 1, 3, 5, 10\}$;
    \item the correct model versus a model that incorporates SEMMS for variable selection.
\end{itemize}

For each configuration, we record the misclassification error, the mean distance fairness measure $\mathrm{MD}$ (Eq.~\eqref{eq:MD}), and the correlation between predicted probabilities and the estimated sensitive attribute.

\begin{figure}[h]
\centering
\includegraphics[width=15 cm]{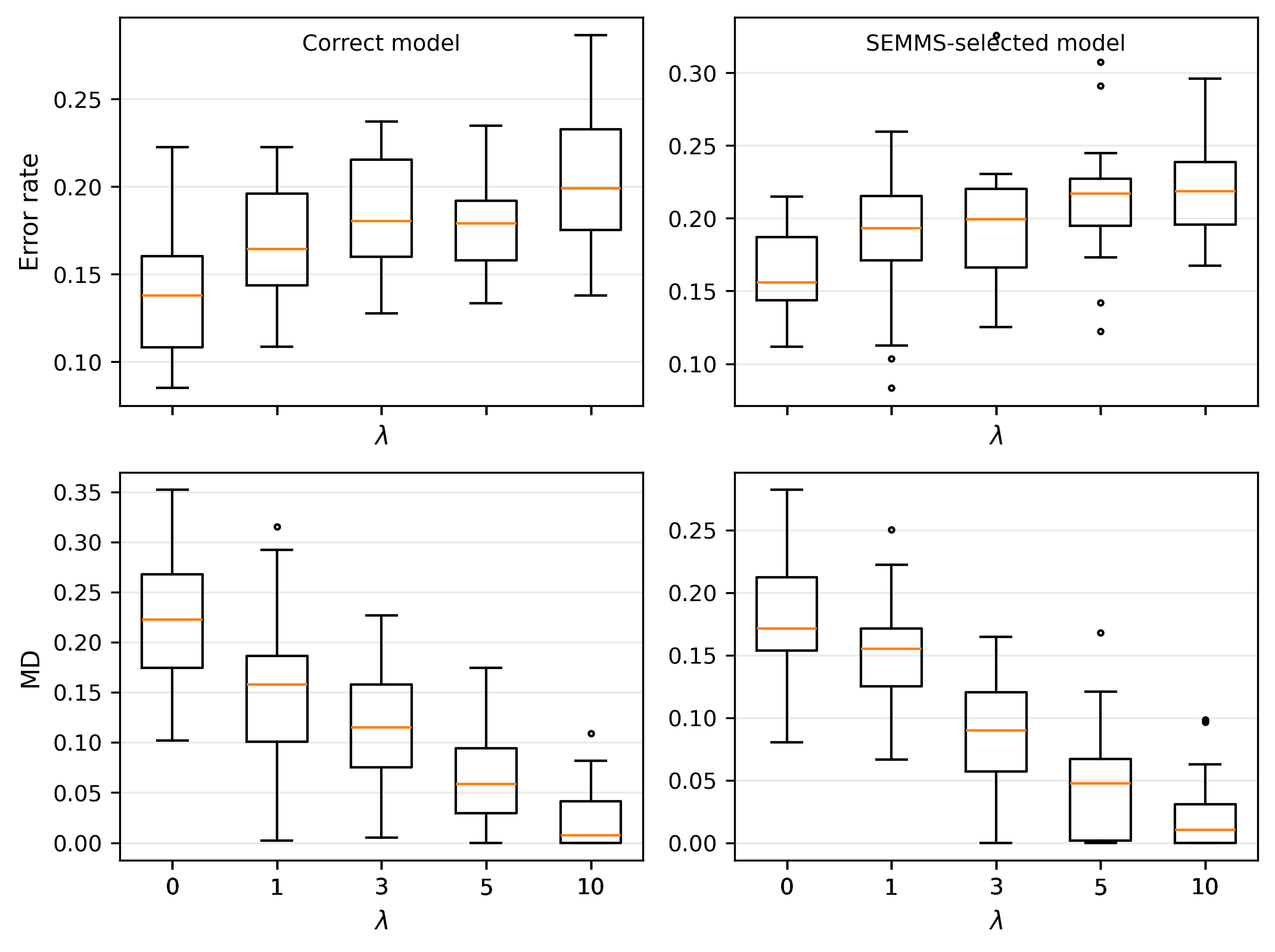}
\caption{Gaussian mixture classification (Model 4): $\mu=2$, error rate and $\mathrm{MD}$ versus $\lambda$, for the correct model and the SEMMS-selected model.}
\label{Figure4}
\end{figure}

\begin{figure}[h]
\centering
\includegraphics[width=15 cm]{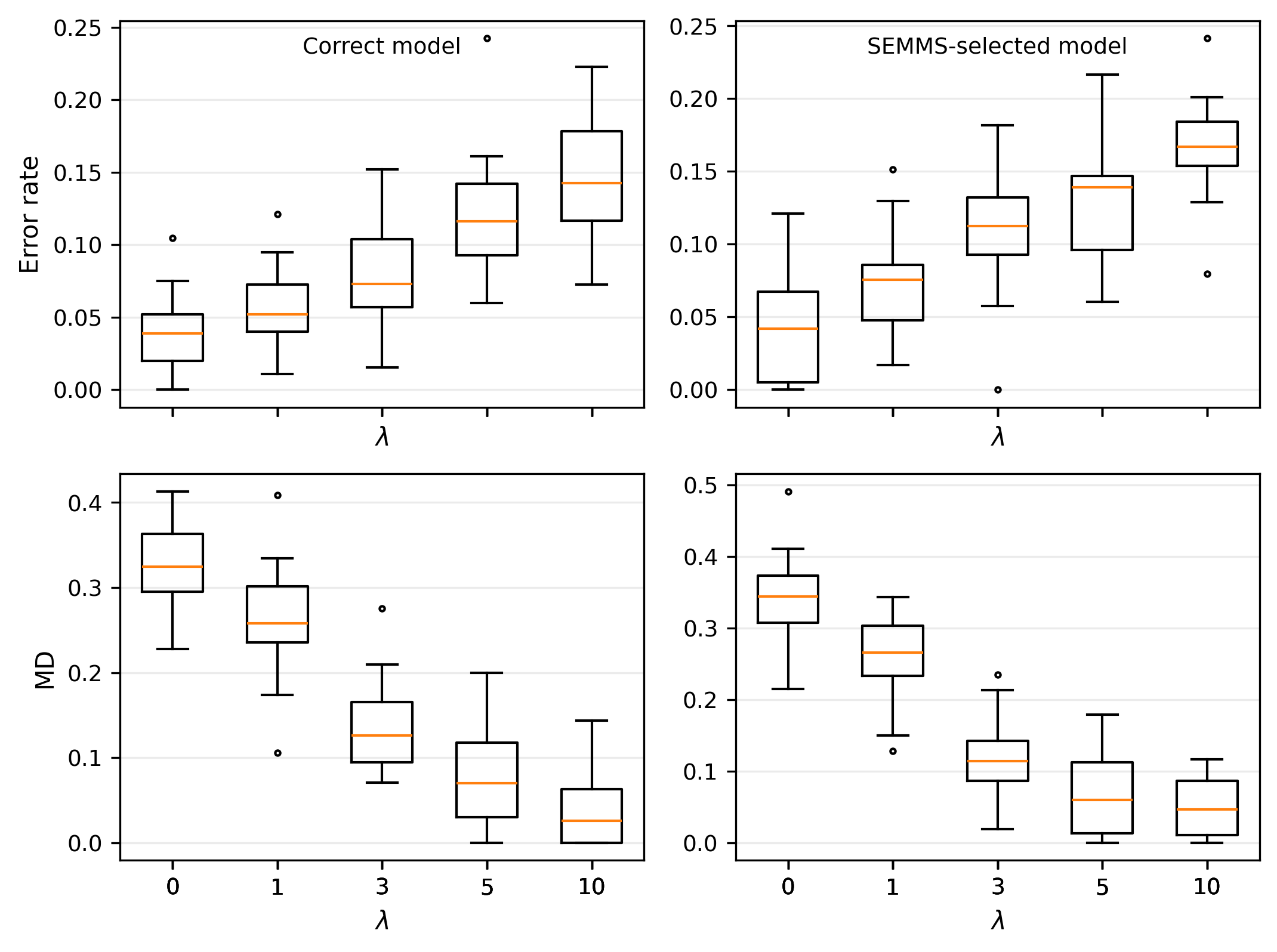}
\caption{Gaussian mixture classification (Model 4): $\mu=6$, error rate and $\mathrm{MD}$ versus $\lambda$, for the correct model and the SEMMS-selected model.}
\label{Figure5}
\end{figure}

In all panels, increasing $\lambda$ leads to a deterioration in the error rate and a decrease in $\mathrm{MD}$, reflecting the trade-off between predictive accuracy and fairness induced by the penalty. This trend is less linear for the model with variable selection, as seen by comparing the left and right columns: SEMMS can either help or hurt, depending on whether it selects mixture-related features that are strongly aligned with $A$.

Comparing $\mu=2$ and $\mu=6$ shows that the latter exhibits stronger fairness effects: when $\mu=6$, the sensitive structure is easier to identify, so increasing $\lambda$ yields a sharper reduction in $\mathrm{MD}$ but also a more pronounced increase in error.

\subsection{Binary classification (categorical)}
\label{subsec:sim-bin-categorical}

Finally, we study binary classification with categorical mixture predictors
using the same product-multinomial structure as in
Section~\ref{subsec:sim-cat-acc}, with $K=2$, $D=3$, and $M_d=2$ for all
$d$. We fix
\[
p=0.5,\qquad n=1000,
\]
and set
\[
\theta_{1,1}=0.1,\qquad \theta_{1,2}=0.9,\qquad
\theta_{2,1}=0.4,\qquad \theta_{2,2}=0.7,
\]
\[
\theta_{1,3}=0.5,\qquad \theta_{2,3}=0.5.
\]
Thus, $X_1$ and $X_2$ carry information about the latent sensitive
attribute, whereas $X_3$ has the same distribution in the two latent groups
and is therefore uninformative about $A$. We consider fairness penalties
\[
\lambda\in\{0,0.6,1.2,1.8,2.4,3.0\}.
\]

For each observation, the binary response is generated according to
\[
Y_i\mid\bX_i\sim\operatorname{Bernoulli}(\pi_i),
\qquad
\pi_i=\frac{\exp(\eta_i)}{1+\exp(\eta_i)},
\]
where $\eta_i$ denotes the outcome linear predictor.

In the first setting, we define
\[
\eta_i=-1+4X_{1i}-X_{2i}-X_{3i},
\]
corresponding to the coefficient vector
\[
\bbeta=(-1,4,-1,-1)^\top,
\]
where the first entry is the intercept. We then generate
\[
Y_i\mid\bX_i\sim
\operatorname{Bernoulli}
\left\{
\frac{\exp(\eta_i)}{1+\exp(\eta_i)}
\right\}
\]
and apply the fairness-aware logistic regression in
Eq.~\eqref{eq:logit-pen}.

Over 20 replications, we record the misclassification error and mean
distance $\mathrm{MD}$ as functions of $\lambda$. As shown in
Figure~\ref{cat_error_MD_1_sim}, increasing $\lambda$ strengthens the
penalty on the dependence between $\hat{\ba}$ and the predicted
probabilities, leading to a higher error rate but a lower $\mathrm{MD}$.
Because the largest coefficient is attached to $X_1$, which is strongly
informative about $A$, the fairness penalty suppresses an important part of
the outcome signal and therefore produces a more substantial loss in
predictive accuracy.

We also observe greater variability in both error and $\mathrm{MD}$ for
larger values of $\lambda$, as the fairness penalty plays an increasingly
dominant role relative to the logistic likelihood. In this setting, the
error rate ranges approximately from $0.10$ to $0.30$, whereas $\mathrm{MD}$
ranges from about $0.5$ to $0.8$.

\begin{figure}[h!]
    \centering
    \begin{subfigure}[b]{0.49\textwidth}
        \centering
        \includegraphics[width=\textwidth]{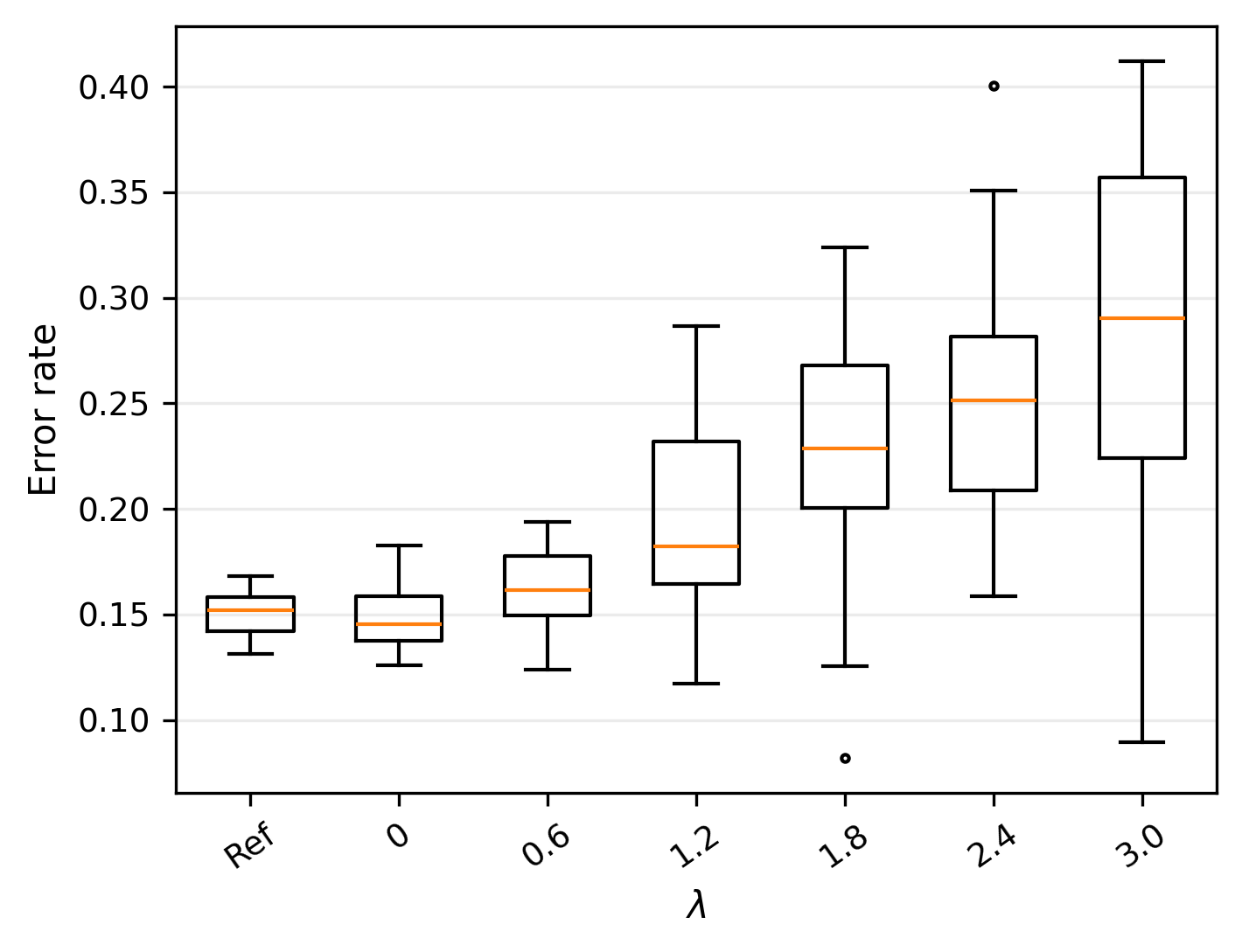}
    \end{subfigure}
    \hfill
    \begin{subfigure}[b]{0.49\textwidth}
        \centering
        \includegraphics[width=\textwidth]{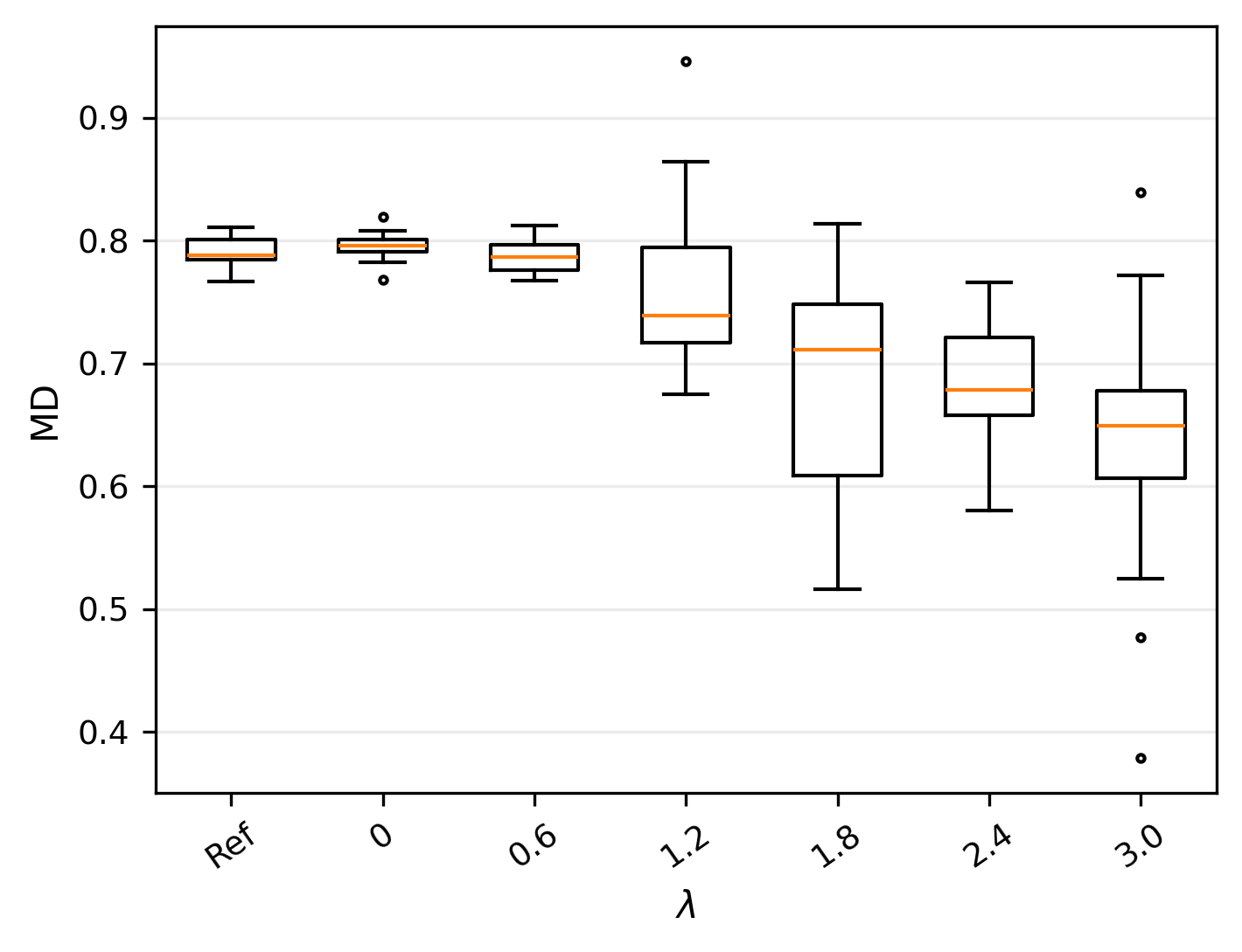}
    \end{subfigure}
    \caption{Categorical mixture classification: error rate and $\mathrm{MD}$
    versus $\lambda$ under the first coefficient configuration,
    $\bbeta=(-1,4,-1,-1)^\top$, where the first entry is the intercept.}
    \label{cat_error_MD_1_sim}
\end{figure}

In the second setting, we assign the largest coefficient to $X_3$, the
predictor that is uninformative about the sensitive attribute. Specifically,
we define
\[
\eta_i=-4+X_{1i}-X_{2i}+4X_{3i},
\]
corresponding to
\[
\bbeta=(-4,1,-1,4)^\top,
\]
and generate
\[
Y_i\mid\bX_i\sim
\operatorname{Bernoulli}
\left\{
\frac{\exp(\eta_i)}{1+\exp(\eta_i)}
\right\}.
\]

The corresponding results are shown in
Figure~\ref{cat_error_MD_2_sim}. The qualitative trade-off between error and
$\mathrm{MD}$ remains as $\lambda$ increases. However, the variability of
the error rate remains relatively small, whereas the variability of
$\mathrm{MD}$ is more pronounced, including for the reference and
unpenalized models.

The error rate remains between approximately $0.18$ and $0.25$, and
$\mathrm{MD}$ does not exceed $0.5$. Because the dominant contribution to
the outcome linear predictor now comes from $X_3$, whose distribution does
not vary with $A$, the fairness penalty can reduce the group-related
component of the predictions without removing most of the predictive
signal. Consequently, stronger fairness penalties can be imposed with a
smaller loss in classification accuracy.

\begin{figure}[h!]
    \centering
    \begin{subfigure}[b]{0.49\textwidth}
        \centering
        \includegraphics[width=\textwidth]{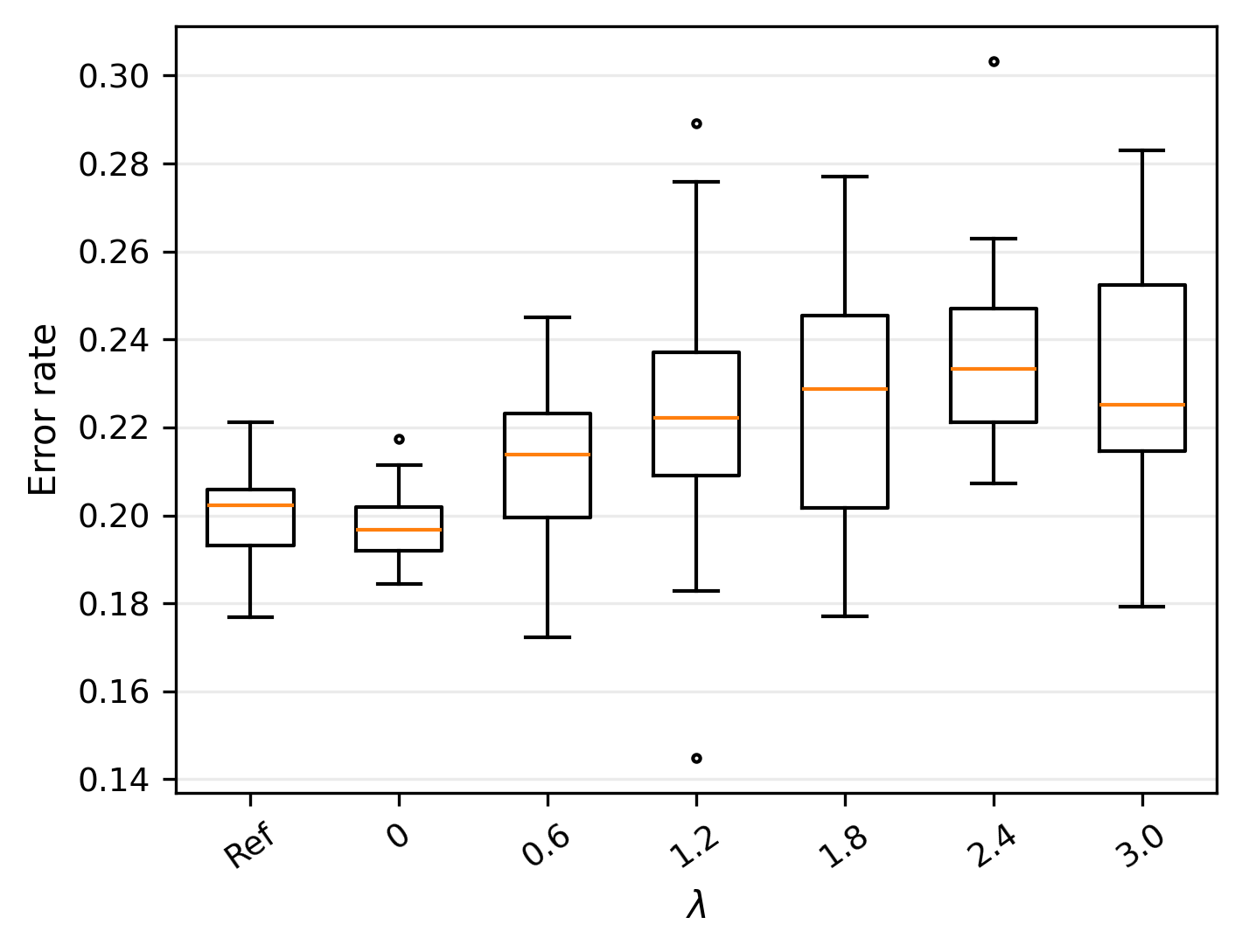}
    \end{subfigure}
    \hfill
    \begin{subfigure}[b]{0.49\textwidth}
        \centering
        \includegraphics[width=\textwidth]{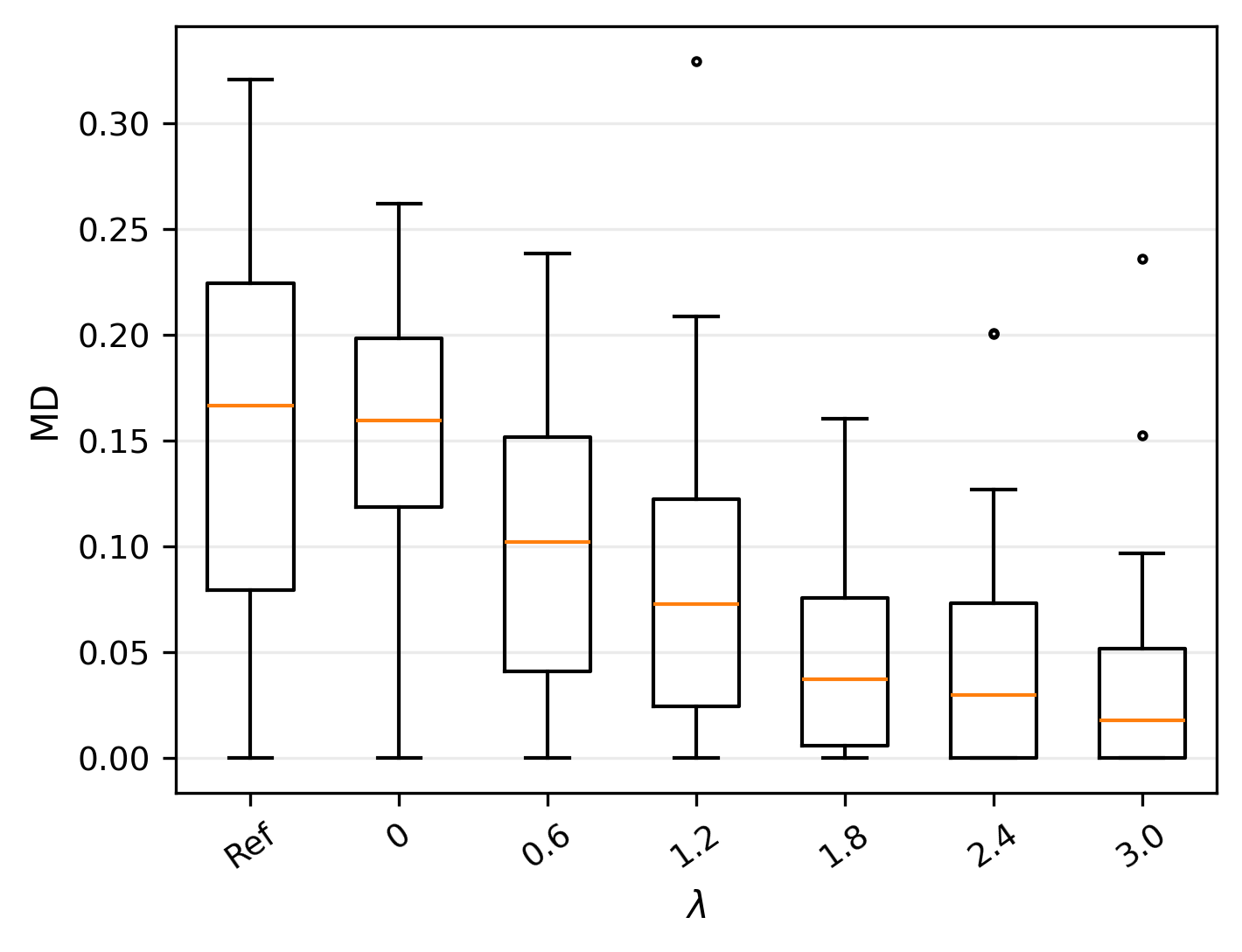}
    \end{subfigure}
    \caption{Categorical mixture classification: error rate and $\mathrm{MD}$
    versus $\lambda$ under the second coefficient configuration,
    $\bbeta=(-4,1,-1,4)^\top$, where the first entry is the intercept.}
    \label{cat_error_MD_2_sim}
\end{figure}
\section{Real Data Applications}
\label{sec:realdata} 

\subsection{Datasets}
\label{subsec:real-datasets}

We evaluate the proposed method on three real datasets: Adult \citep{adult1996}, COMPAS \citep{angwin2016machinebias}, and ARRHYTHMIA \citep{guvenir1997arrhythmia}.

\begin{itemize}
    \item \textbf{Adult} \citep{adult1996}: This dataset contains 45,221 records of personal yearly income, with a binary label indicating whether the yearly salary is over or under \$50K. Gender is considered a sensitive attribute.
    \item \textbf{COMPAS}\citep{angwin2016machinebias}: This dataset assesses the likelihood of recidivism within a future period, containing 11,750 criminal records collected in the U.S. The race of each defendant is treated as a sensitive attribute. 
    \item \textbf{ARRHYTHMIA}\citep{guvenir1997arrhythmia}: This dataset, which requires variable selection, aims to distinguish between the presence and absence of cardiac arrhythmia. It contains 452 instances and 279 attributes, 206 of which are numerical, while the rest are nominal. Gender is a sensitive attribute. We first select variables using SEMMS and then apply our general framework.
\end{itemize}

\subsection{Experimental Setup}
\label{subsec:real-setup}

For the Adult and COMPAS experiments, we follow the experimental setup and related-feature baselines in \citet{zhao2022fairrf}. We optimize the objective using Adam \citep{kingma2015adam}.

For all datasets, we randomly split the data into training and test sets using a 7:3 ratio. Feature screening, mixture fitting, and tuning-parameter selection are performed using the training portion only; the test portion is reserved for final evaluation. Unless otherwise noted, entries are reported as mean $\pm$ standard deviation over five independent random splits.

For the Adult dataset, we select age (numeric), relationship status (categorical),
and marital status (categorical) as relevant variables for estimating $\hat\bA$,
following \citet{zhao2022fairrf}. Since this relevant set contains both continuous and
categorical predictors, we fit a hybrid Gaussian--categorical $K$-component
mixture model as in Section~\ref{subsec:latentA}.

Let $\bx_i^{\,\ba} = (\tilde{\bx}_i,\bx_i') \in \mathbb{R}^{D}\times\mathbb{R}^{M}$,
where $\tilde{\bx}_i$ denotes the categorical part (with $m_d$ levels for the
$d$th predictor) and $\bx_i'$ denotes the continuous part. We assume conditional
independence across predictors and a shared diagonal covariance matrix for the
continuous part, $\bSigma=\diag(\sigma_1^2,\dots,\sigma_M^2)$.
For component $k=1,\dots,K$, the component density factorizes as
\[
f_k(\bx_i^{\,\ba};\btheta_k)
=
\left\{
\prod_{d=1}^{D}\prod_{\ell=1}^{m_d}
\pi_{kd\ell}^{\,z_{id\ell}}
\right\}
\left\{
\prod_{m=1}^{M}\phi(x'_{i,m};\mu_{k,m},\sigma_m^2)
\right\},
\]
where $z_{id\ell}=\mathbb{I}\{\tilde{x}_{i,d}\text{ takes level }\ell\}$,
$\bpi_{kd}=(\pi_{kd1},\dots,\pi_{kd m_d})^\top$ is the probability vector for
the $d$th categorical predictor in component $k$, and $\phi(\cdot;\mu_{k,m},\sigma_m^2)$
is the univariate Gaussian density for the $m$th continuous coordinate. The
parameter $\btheta_k$ collects $\{\bpi_{kd}\}_{d=1}^D$ and $\{\mu_{k,m}\}_{m=1}^M$
(and $\{\sigma_m^2\}_{m=1}^M$ if updated).

Given $(\bp,\btheta_1,\dots,\btheta_K)$, the posterior class probabilities are
\[
a_{ik}=\Pr(A_i=k\mid \bx_i^{\,\ba})
=
\frac{p_k f_k(\bx_i^{\,\ba};\btheta_k)}
{\sum_{\ell=1}^K p_\ell f_\ell(\bx_i^{\,\ba};\btheta_\ell)},
\qquad i=1,\dots,n,\; k=1,\dots,K,
\]
and we estimate them by substituting the EM maximum likelihood estimates, as in
\eqref{eq:posterior-a}--\eqref{eq:em-M-theta}. In particular, the E-step computes
$\hat a_{ik}^{(t+1)}$ as in \eqref{eq:em-E}. In the M-step, because $f_k$ has the
product form above, the updates separate into the categorical and Gaussian parts:
the multinomial updates follow \eqref{eq:em-multinomial-pi} and the Gaussian mean
updates follow the diagonal-covariance specialization of \eqref{eq:em-gaussian-mu}
(Section~\ref{subsubsec: Gaussian est}). Collecting
$\hat\ba_i=(\hat a_{i1},\dots,\hat a_{iK})^\top$ yields the estimated sensitive
attribute matrix $\hat\bA=[\hat\ba_1,\dots,\hat\ba_n]^\top$.

For the COMPAS dataset, to satisfy the necessary counting condition discussed in Section~\ref{subsec:identifiability}, we group race into two levels: \emph{African} and \emph{Non-African}. Unlike \citet{zhao2022fairrf}, we observe that the decile score (10 levels) and the score text (3 levels) convey essentially equivalent information. Including both would fail this screening rule and make generic identifiability doubtful, so we retain only the decile score (10 levels) together with age (a 3-level factor), crime rate, and sex as relevant predictors to estimate \(\hat{\ba}\).

Once \(\hat{\ba}\) has been estimated for Adult and COMPAS, we construct residualized predictors \(\hat{\bU}\) as in Section~\ref{subsec:residual} and fit the fairness-aware logistic regression model in Eq.~\eqref{eq:logit-pen}, optimizing it via Adam.

For the ARRHYTHMIA dataset, SEMMS selects the 5th, 91st, 179th, 199th, 224th, and 279th predictors. These correspond to QRS duration-related measurements from different leads/channels: V1, S' wave, AVL, DII, AVR, and V6.\footnote{The QRS duration is the time interval between the start of the Q wave and the end of the S wave on an electrocardiogram (ECG), reflecting the duration of ventricular depolarization. It is typically measured in milliseconds, with a normal range of 80–120 ms. Prolonged QRS duration can indicate conduction abnormalities in the ventricles.}  
All selected predictors are numeric, so we fit a Gaussian mixture to this reduced set and apply the EM algorithm as in Section~\ref{subsec:latentA}. The screening procedure identifies the 199th variable as the most relevant for estimating \(\hat{\ba}\).

\subsection{Experimental Results}
\label{subsec:real-results}

\subsubsection{Accuracy of predicting \texorpdfstring{$\ba$}{a}}

We first assess the quality of sensitive attribute estimation by comparing the AUC achieved by our EM-based approach to that of FairWS and FairWS+MI \citep{zhu2022learning}, which are based on a probabilistic generative framework. As summarized in Table~\ref{table1-a}, our method attains an AUC of 0.70 on Adult, which is in the same general range as the FairWS baselines, although it is lower than their reported AUCs. For COMPAS, the AUC is about 0.62, and for the more complex ARRHYTHMIA dataset, the AUC is approximately 0.57. Overall, these results indicate that the latent sensitive attribute can be recovered with useful, though imperfect, accuracy, which is sufficient to support the downstream fairness-adjustment step.

\begin{table}[H]
\begin{center}
 \begin{tabular}{c|ccc} 
 \hline
  & FairWS  & FairWS + MI & EM algorithm \\
 \hline
 Adult       & 0.76 & 0.77 & 0.70 \\ 
 COMPAS      &  0.57    &  0.59    & 0.62 \\
 ARRHYTHMIA  &  0.63    &    0.63  & 0.57 \\
 \hline
\end{tabular}
\caption{AUC for estimating the sensitive attribute \(\hat{\ba}\) on three benchmark datasets.}
\label{table1-a}
\end{center}
\end{table}

\subsubsection{Accuracy and fairness of predicting \texorpdfstring{$\by$}{y}}

We next compare our approach against several baselines, including the unconstrained model (Vanilla), the oracle fairness-constrained model using the true sensitive attribute (ConstrainS), adversarially reweighted learning (ARL) \citep{lahoti2020arl}, fair class balancing via synthetic oversampling (KSMOTE) \citep{yan2020faircb}, the related-feature removal baseline (RemoveR), and FairRF \citep{zhao2022fairrf}. Following \citet{zhao2022fairrf}, ConstrainS and RemoveR are included as oracle and feature-removal baselines, respectively.

For each fitted classifier, we evaluate predictive performance using classification accuracy (ACC), and fairness using gap-based metrics related to equality of opportunity and demographic parity \citep{hardt2016equality,feldman2015disparate}. Specifically, we report:
\begin{itemize}
    \item \textbf{ACC}: classification accuracy,
    \item \(\Delta_{\mathrm{EO}}\): the maximum absolute difference in true positive and false positive rates across groups (equalized odds gap),
    \item \(\Delta_{\mathrm{DP}}/\mathrm{MD}\): the demographic parity gap and, in our case, the mean distance (\(\mathrm{MD}\), defined in Eq.~\eqref{eq:MD}) as a proxy for group-wise performance disparity.
\end{itemize}

Table~\ref{tab:adult_comparison} shows the results for the Adult dataset. With a suitable choice of \(\lambda\), our method reduces the fairness gap to near zero (about a 50\% improvement in \(\Delta_{\mathrm{DP}}/\mathrm{MD}\)) while incurring only about a 1\% loss in accuracy relative to the best-performing baselines.

\begin{table}[h!]
    \centering
    \caption{Predictive accuracy and fairness on Adult. Higher ACC is better, whereas lower gaps are better. Boldface marks the best value in each column; ties are all bolded.}
    \label{tab:adult_comparison}
    \begin{tabular}{lccc}
        \hline
        \textbf{Method} & \textbf{ACC} & $\boldsymbol{\Delta}_{\mathbf{EO}}$ & $\boldsymbol{\Delta}_{\mathbf{DP}}/\mathbf{MD}$ \\
        \hline
        Vanilla                    & 0.856 $\pm$ 0.001 & 0.046 $\pm$ 0.006 & 0.089 $\pm$ 0.005 \\
        ConstrainS                 & 0.845 $\pm$ 0.002 & 0.040 $\pm$ 0.004 & \textbf{0.058 $\pm$ 0.003} \\
        ARL                        & \textbf{0.861 $\pm$ 0.003} & 0.034 $\pm$ 0.012 & 0.141 $\pm$ 0.008 \\
        KSMOTE                     & 0.560 $\pm$ 0.002 & 0.141 $\pm$ 0.031 & 0.120 $\pm$ 0.022 \\
        RemoveR                    & 0.801 $\pm$ 0.010 & 0.124 $\pm$ 0.004 & 0.071 $\pm$ 0.002 \\
        FairRF \citep{zhao2022fairrf} & 0.832 $\pm$ 0.001 & \textbf{0.025 $\pm$ 0.009} & 0.066 $\pm$ 0.004 \\
        FairRF (replication)       & 0.846 $\pm$ 0.009 & 0.041 $\pm$ 0.004 & 0.075 $\pm$ 0.006 \\
        Proposed ($\lambda=0.3$) & 0.837 $\pm$ 0.039 & \textbf{0.025 $\pm$ 0.007} & \textbf{0.058 $\pm$ 0.002} \\
        \hline
    \end{tabular}
\end{table}

For COMPAS (Table~\ref{tab:compas_comparison}), the proposed method attains the best equalized-odds gap (tied with our FairRF replication) and a demographic-parity gap close to the oracle ConstrainS baseline, while retaining accuracy within 0.003 of the best-performing method.

\begin{table}[h!]
\centering
\caption{Predictive accuracy and fairness on COMPAS. Higher ACC is better, whereas lower gaps are better. Boldface marks the best value in each column; ties are all bolded.}
\label{tab:compas_comparison}
\begin{tabular}{lccc}
\hline
\textbf{Method} & \textbf{ACC} & $\boldsymbol{\Delta}_{\mathbf{EO}}$ & $\boldsymbol{\Delta}_{\mathbf{DP}}/\mathbf{MD}$ \\
\hline
Vanilla                    & \textbf{0.681 $\pm$ 0.004} & 0.242 $\pm$ 0.021 & 0.171 $\pm$ 0.015 \\
ConstrainS                 & 0.674 $\pm$ 0.002 & 0.154 $\pm$ 0.032 & \textbf{0.122 $\pm$ 0.031} \\
ARL                        & 0.672 $\pm$ 0.023 & 0.197 $\pm$ 0.042 & 0.286 $\pm$ 0.033 \\
KSMOTE                     & 0.601 $\pm$ 0.021 & 0.203 $\pm$ 0.042 & 0.151 $\pm$ 0.023 \\
RemoveR                    & 0.595 $\pm$ 0.024 & 0.205 $\pm$ 0.049 & 0.185 $\pm$ 0.024 \\
FairRF \citep{zhao2022fairrf} & 0.661 $\pm$ 0.009 & 0.166 $\pm$ 0.022 & 0.143 $\pm$ 0.021 \\
FairRF (replication)       & 0.676 $\pm$ 0.014 & \textbf{0.130 $\pm$ 0.013} & 0.135 $\pm$ 0.007 \\
Proposed ($\lambda=0.3$) & 0.678 $\pm$ 0.010 & \textbf{0.130 $\pm$ 0.017} & 0.132 $\pm$ 0.008 \\
\hline
\end{tabular}
\end{table}     

\subsubsection{Accuracy–fairness trade-off}
\label{subsec:real-tradeoff}

Finally, we visualize the trade-off between classification error and mean distance as the fairness penalty \(\lambda\) varies in Eq.~\eqref{eq:logit-pen}. As \(\lambda\) increases, the penalty on dependence between \(\hat{\bA}\) and the predictions becomes stronger, leading to more similar group-wise performance (lower \(\mathrm{MD}\)) but also higher classification error.

\begin{figure}[h!]
    \centering
    \begin{subfigure}[b]{0.49\textwidth}
        \centering
        \includegraphics[width=\textwidth]{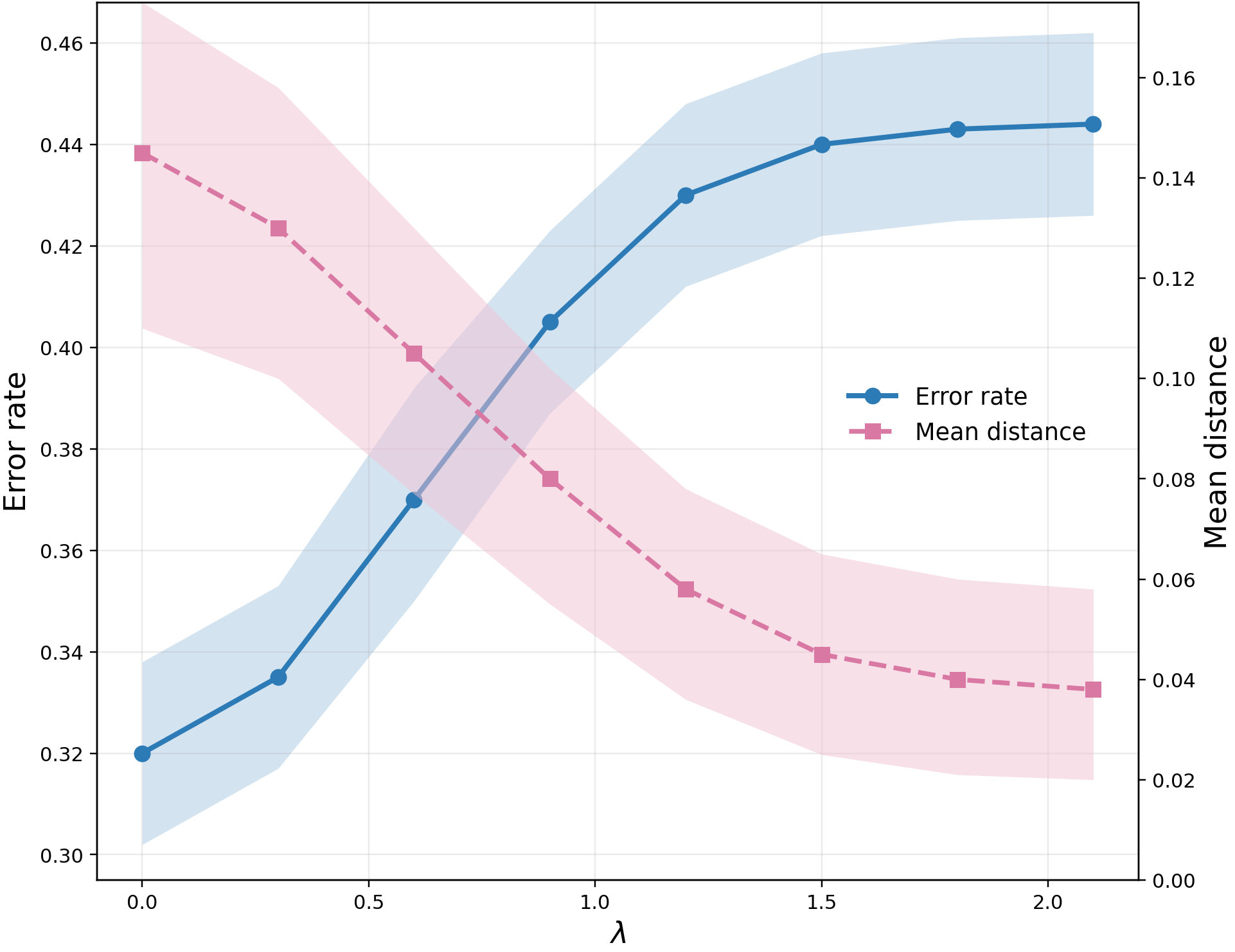}
    \end{subfigure}
    \hfill
    \begin{subfigure}[b]{0.49\textwidth}
        \centering
        \includegraphics[width=\textwidth]{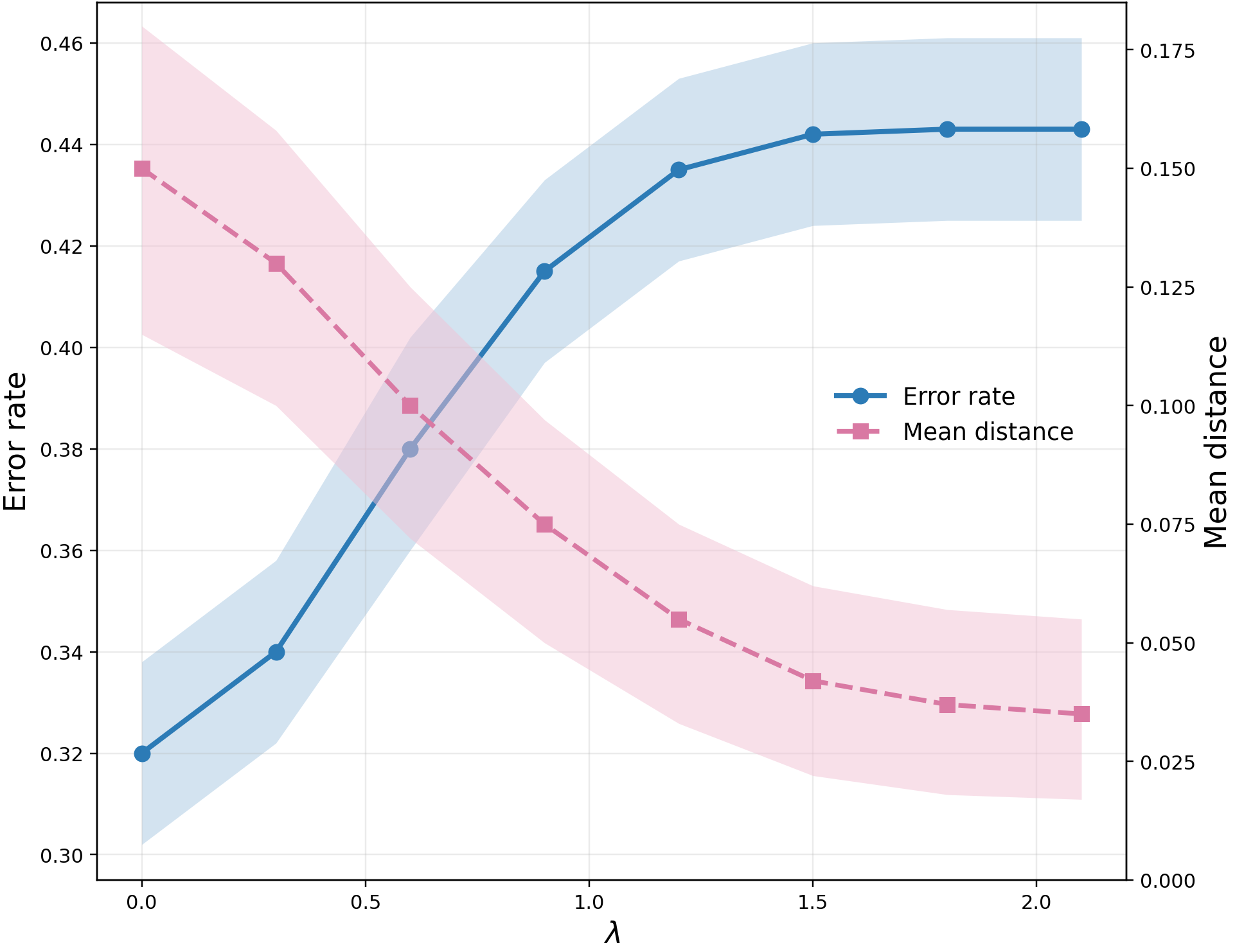}
    \end{subfigure}
    \caption{Error rate and mean distance versus $\lambda$: left, Adult; right, COMPAS.}
    \label{adult_lambda}
\end{figure}

For ARRHYTHMIA, we observe a similar pattern: increasing \(\lambda\) moves the classifier along an explicit accuracy--fairness trade-off curve, with \(\mathrm{MD}\) decreasing and the error rate gradually increasing.

\begin{figure}[h!]
    \centering
    \begin{subfigure}[b]{0.52\textwidth}
        \centering
        \includegraphics[width=\textwidth]{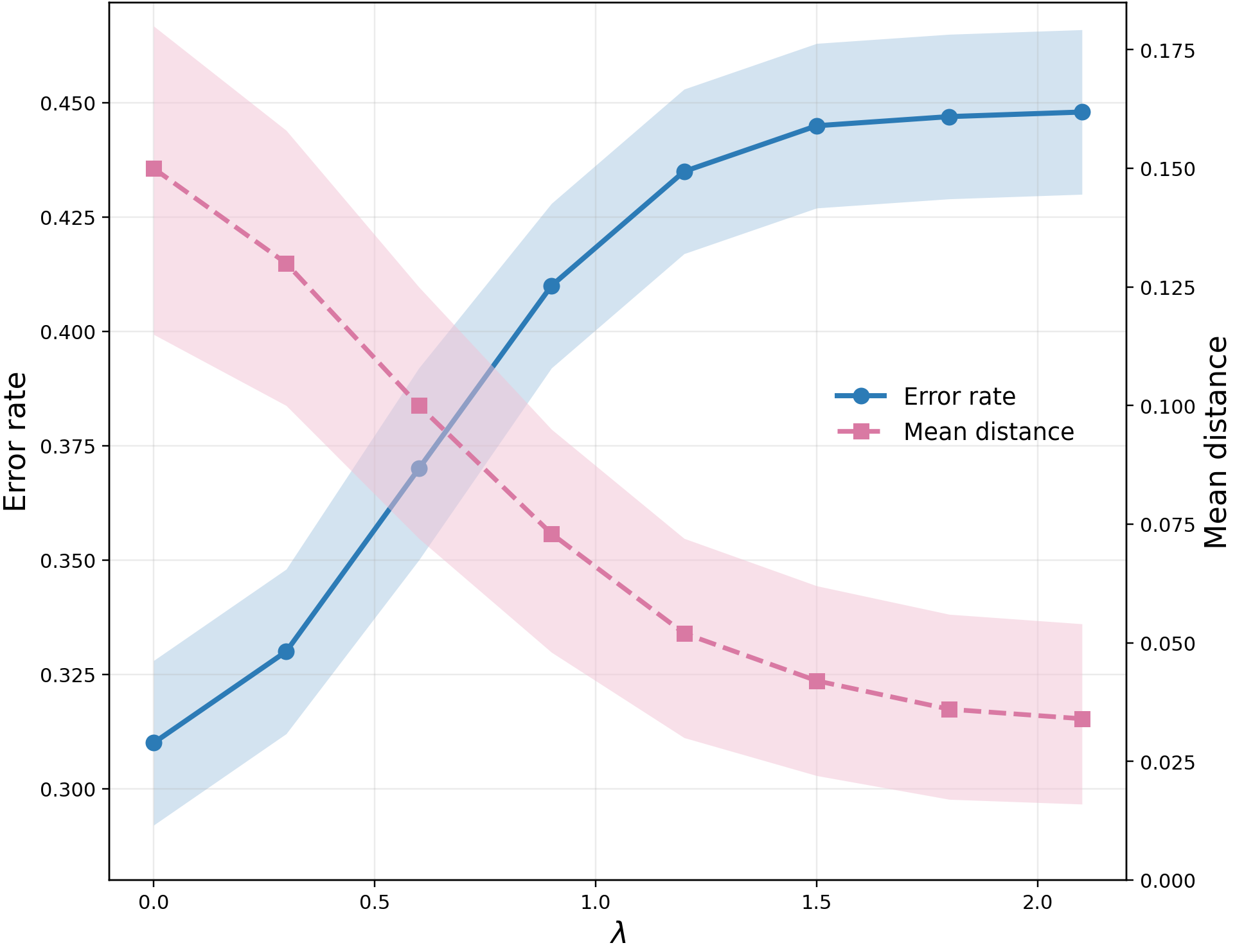}
        \caption{Error rate and mean distance versus $\lambda$ for ARRHYTHMIA.}
    \end{subfigure}
\end{figure}

\section{Conclusion}
\label{sec:conclusion}

We introduced a framework for fairness-aware generalized linear modeling when sensitive attributes are unavailable and the predictor dimension may be large. The procedure estimates latent group membership through finite mixture models, residualizes predictors with respect to the resulting posterior probabilities, and controls dependence between predictions and the estimated sensitive structure. Continuous outcomes are handled through an $R^2$ constraint, binary outcomes through a correlation-based penalty, and high-dimensional predictors through integration with SEMMS variable selection.

The theoretical analysis clarifies both the possibilities and the limitations of this strategy. Parameter counting provides only a necessary screening condition for categorical-mixture identifiability, while stronger generic identifiability conditions are needed for formal recovery. Conditional on identifiable and sufficiently separated components, the posterior-mode classifier recovers the latent groups with increasing accuracy. The regression calculations further quantify the predictive information lost when group-associated variation is removed, making the accuracy--fairness trade-off explicit.

The simulations agree with these results, and the Adult, COMPAS, and Arrhythmia applications show that the proposed procedure can reduce group disparities while retaining useful predictive performance. These findings should nevertheless be interpreted with care: fairness is assessed relative to an estimated sensitive attribute, so estimation error, model misspecification, and weak component separation may all limit the protection achieved for the true unobserved groups.

Future work should develop uncertainty-aware fairness constraints that propagate error in the latent-group estimates, establish guarantees under mixture misspecification, and extend the framework to richer mixed-data and nonlinear models. A further priority is to study how proxy selection and fairness tuning can be performed with valid post-selection inference and robust out-of-sample guarantees.

\appendix

\section{Proofs}
\label{app:proofs-theory}

\subsection{Gaussian mixture classification}

\subsubsection{Proof of Theorem~\ref{thm:gaussian-multi}}
\label{app:proof-gaussian-multi}
\begin{proof}
For an observed vector $x$, Bayes' rule gives the posterior probability of class $k$ as
\[
\Pr(A=k\mid X=x)
\propto
p_k\,\phi_{p_a}(x;\bmu_k,\bSigma_e).
\]
Therefore the posterior-mode classifier assigns $x$ to the class maximizing
\[
\log p_k-\tfrac12(x-\bmu_k)^\top\bSigma_e^{-1}(x-\bmu_k),
\]
which is exactly the definition of the decision region $\mathcal D_k$. Conditioning on the true class and integrating over the corresponding decision region yields
\[
\Pr(\hat A=A)
=
\sum_{k=1}^K p_k\Pr_k(X\in\mathcal D_k)
=
\sum_{k=1}^K p_k\int_{\mathcal D_k}\phi_{p_a}(x;\bmu_k,\bSigma_e)\,dx.
\]
This proves the exact expression.

It remains to show that the accuracy tends to one as $\delta_{\min}\to\infty$. Fix two distinct classes $k$ and $\ell$, and define the log-posterior difference
\[
D_{k\ell}(X)
=
\log\frac{p_k\phi_{p_a}(X;\bmu_k,\bSigma_e)}{p_\ell\phi_{p_a}(X;\bmu_\ell,\bSigma_e)}.
\]
Under the event $A=k$, we may write $X=\bmu_k+E$ with $E\sim N(0,\bSigma_e)$. Substituting this into $D_{k\ell}(X)$ and simplifying gives
\[
D_{k\ell}(X)
=
\log\frac{p_k}{p_\ell}
+\frac12(\bmu_k-\bmu_\ell)^\top\bSigma_e^{-1}(\bmu_k-\bmu_\ell)
+(\bmu_k-\bmu_\ell)^\top\bSigma_e^{-1}E.
\]
Hence, conditional on $A=k$,
\[
D_{k\ell}(X)\sim N\!\left(\log\frac{p_k}{p_\ell}+\frac{\delta_{k\ell}^2}{2},\;\delta_{k\ell}^2\right).
\]
Class $\ell$ defeats class $k$ only when $D_{k\ell}(X)\le 0$, so
\[
\Pr_k(\ell\text{ defeats }k)
=
\Phi\!\left(\frac{\log(p_\ell/p_k)}{\delta_{k\ell}}-\frac{\delta_{k\ell}}{2}\right).
\]
Because each class probability is positive, $\log(p_\ell/p_k)$ is finite. As $\delta_{k\ell}\to\infty$, the argument of $\Phi$ tends to $-\infty$, and therefore
\[
\Pr_k(\ell\text{ defeats }k)\to 0.
\]
Now
\[
\Pr_k(\hat A\neq k)
=
\Pr_k\Bigl(\bigcup_{\ell\ne k}\{\ell\text{ defeats }k\}\Bigr)
\le
\sum_{\ell\ne k}\Pr_k(\ell\text{ defeats }k),
\]
so if $\delta_{\min}\to\infty$, the right-hand side tends to zero uniformly in $k$. Thus $\Pr_k(\hat A=k)\to 1$ for every $k$, and averaging with respect to $p_k$ yields
\[
\Pr(\hat A=A)\to 1.
\]
\end{proof}

\subsubsection{Proof of Theorem~\ref{thm:gaussian-uni}}
\label{app:proof-gaussian-uni}
\begin{proof}
The exact formula is a one-dimensional specialization of the previous argument. Indeed, for each class $k$, the posterior-mode rule assigns $x$ to class $k$ exactly when $x\in\mathcal I_k$, where
\[
\mathcal I_k
=
\left\{x:
\log p_k-\frac{(x-\mu_k)^2}{2\sigma_e^2}
\ge
\log p_\ell-\frac{(x-\mu_\ell)^2}{2\sigma_e^2}
\ \text{for every }\ell\ne k
\right\}.
\]
Therefore
\[
\Pr(\hat A=A)
=
\sum_{k=1}^K p_k\Pr_k(X\in\mathcal I_k)
=
\sum_{k=1}^K p_k\int_{\mathcal I_k}
\frac{1}{\sigma_e}\phi\!\left(\frac{x-\mu_k}{\sigma_e}\right)dx.
\]
This proves the exact accuracy formula.

We next derive a sufficient separation condition ensuring accuracy at least $1-\alpha$. Fix $k\neq \ell$ and, conditional on $A=k$, write
\[
X=\mu_k+\sigma_e Z,\qquad Z\sim N(0,1).
\]
Consider again the log-posterior difference between classes $k$ and $\ell$:
\[
D_{k\ell}(X)
=
\log\frac{p_k\phi((X-\mu_k)/\sigma_e)}{p_\ell\phi((X-\mu_\ell)/\sigma_e)}.
\]
A direct simplification gives
\[
D_{k\ell}(X)
=
\log\frac{p_k}{p_\ell}+\frac{\delta_{k\ell}^2}{2}+\delta_{k\ell}Z,
\]
where $\delta_{k\ell}=|\mu_k-\mu_\ell|/\sigma_e$. Hence, under $A=k$,
\[
D_{k\ell}(X)\sim N\!\left(\log\frac{p_k}{p_\ell}+\frac{\delta_{k\ell}^2}{2},\;\delta_{k\ell}^2\right).
\]
The probability that class $\ell$ defeats class $k$ is therefore
\[
\Pr_k(\ell\text{ defeats }k)
=
\Phi\!\left(\frac{\log(p_\ell/p_k)}{\delta_{k\ell}}-\frac{\delta_{k\ell}}{2}\right).
\]
Since $p_\ell/p_k\le r_p=p_{\max}/p_{\min}$ and $\delta_{k\ell}\ge \delta_{\min}$, we obtain the uniform bound
\[
\Pr_k(\ell\text{ defeats }k)
\le
\Phi\!\left(\frac{\log r_p}{\delta_{\min}}-\frac{\delta_{\min}}{2}\right).
\]
If we require this bound to be at most $\alpha/(K-1)$, then, by the definition of
\[
z_{\alpha,K}=\Phi^{-1}\!\left(1-\frac{\alpha}{K-1}\right),
\]
it is enough that
\[
\frac{\delta_{\min}}{2}-\frac{\log r_p}{\delta_{\min}}
\ge z_{\alpha,K}.
\]
Multiplying by $\delta_{\min}>0$ gives the quadratic inequality
\[
\delta_{\min}^2-2z_{\alpha,K}\delta_{\min}-2\log r_p\ge 0.
\]
Solving for the positive root yields the sufficient condition
\[
\delta_{\min}\ge z_{\alpha,K}+\sqrt{z_{\alpha,K}^2+2\log r_p}.
\]
Under this condition, for each fixed class $k$,
\[
\Pr_k(\hat A\neq k)
\le
\sum_{\ell\ne k}\Pr_k(\ell\text{ defeats }k)
\le
(K-1)\cdot\frac{\alpha}{K-1}=\alpha,
\]
by a union bound over the $K-1$ competing classes. Hence $\Pr_k(\hat A=k)\ge 1-\alpha$ for every $k$, and averaging over $k$ gives
\[
\Pr(\hat A=A)\ge 1-\alpha.
\]
\end{proof}

\subsection{Categorical mixture classification}

\subsubsection{Proof of Theorem~\ref{thm:categorical-general}}
\label{app:proof-cat-general}
\begin{proof}
For any category pattern $j^*=(j_1^*,\ldots,j_D^*)$, denote the corresponding event by
\[
X=j^*\quad\Longleftrightarrow\quad X_d=j_d^*\ \text{for }d=1,\ldots,D.
\]
Under class $k$, the product-multinomial assumption implies
\[
\Pr(X=j^*\mid A=k)=\prod_{d=1}^D\theta_{k,d,j_d^*}.
\]
By Bayes' rule, the posterior score for class $k$ at the pattern $j^*$ is proportional to
\[
p_k\Pr(X=j^*\mid A=k)=p_k\prod_{d=1}^D\theta_{k,d,j_d^*}.
\]
Thus the posterior-mode classifier assigns $j^*$ to class $k$ exactly when, for every $\ell\neq k$,
\[
p_k\prod_{d=1}^D\theta_{k,d,j_d^*}
>
p_\ell\prod_{d=1}^D\theta_{\ell,d,j_d^*},
\]
or equivalently,
\[
\prod_{d=1}^D\frac{\theta_{k,d,j_d^*}}{\theta_{\ell,d,j_d^*}}>
\frac{p_\ell}{p_k}.
\]
(If ties occur, they are resolved by the fixed tie-breaking rule in the theorem statement.)

Conditioning on $A=k$ and summing over all patterns assigned to class $k$, we obtain
\[
\Pr(\hat A=k\mid A=k)
=
\sum_{\text{all combinations }j^*}
\left[
\prod_{\ell\neq k}\mathbbm{1}\!\left\{
\prod_{d=1}^D\frac{\theta_{k,d,j_d^*}}{\theta_{\ell,d,j_d^*}}>
\frac{p_\ell}{p_k}
\right\}
\prod_{d=1}^D\theta_{k,d,j_d^*}
\right].
\]
Finally, averaging over the latent class $A$ gives
\[
\Pr(\hat A=A)
=
\sum_{k=1}^K p_k\Pr(\hat A=k\mid A=k),
\]
which is exactly the formula stated in the theorem.
\end{proof}

\subsubsection{Proof of Theorem~\ref{thm:categorical-binary}}
\label{app:proof-cat-binary}
\begin{proof}
Because $X$ is binary, correct classification can be computed by conditioning on the two possible observed values.

When $X=1$, the posterior-mode rule predicts class $A=1$ if and only if
\[
p\theta_1\ge (1-p)\theta_2.
\]
This event is encoded by the indicator
\[
m=\mathbbm{1}\!\left(\frac{\theta_1}{\theta_2}\ge\frac{1-p}{p}\right).
\]
Likewise, when $X=0$, the classifier predicts $A=1$ if and only if
\[
p(1-\theta_1)\ge (1-p)(1-\theta_2),
\]
which is encoded by
\[
n=\mathbbm{1}\!\left(\frac{1-\theta_1}{1-\theta_2}\ge\frac{1-p}{p}\right).
\]

Now enumerate the four possible $(A,X)$ combinations.
\begin{itemize}
    \item If $A=1$ and $X=1$, the observation is classified correctly exactly when $m=1$. This contributes
    \[
    p\,\theta_1\,m.
    \]
    \item If $A=1$ and $X=0$, the observation is classified correctly exactly when $n=1$. This contributes
    \[
    p(1-\theta_1)n.
    \]
    \item If $A=0$ and $X=1$, correct classification means predicting class $0$, which happens exactly when $m=0$. This contributes
    \[
    (1-p)\theta_2(1-m).
    \]
    \item If $A=0$ and $X=0$, correct classification means predicting class $0$, which happens exactly when $n=0$. This contributes
    \[
    (1-p)(1-\theta_2)(1-n).
    \]
\end{itemize}
Summing the four contributions gives
\begin{align*}
\Pr(\hat A=A)
&=p\theta_1m+p(1-\theta_1)n+(1-p)\theta_2(1-m)+(1-p)(1-\theta_2)(1-n)\\
&=(m-n)\bigl(p\theta_1-(1-p)\theta_2\bigr)+n(2p-1)+(1-p),
\end{align*}
which proves the stated formula.
\end{proof}

\subsection{$R^2$ after residualization}

\subsubsection{Proof of Theorem~\ref{thm:R2-general}}
\label{app:proof-R2-general}
\begin{proof}
Write the linear signal as
\[
S=\bbeta_a^\top\bX^a+\bbeta_z^\top\bX^z.
\]
Under Model~\ref{model:3.1},
\[
\bX^a=\bmu^\top A+E,
\]
where $E$ has covariance matrix $\bSigma_e$ and is independent of $A$. Therefore
\[
S=\bbeta_a^\top\bmu^\top A+\bbeta_a^\top E+\bbeta_z^\top\bX^z.
\]
Since $Y=\beta_0+S+\varepsilon$ with $\varepsilon$ independent of the predictors, the population coefficient of determination when regressing $Y$ on the full predictor vector equals
\[
R_X^2=\frac{\Var(S)}{\Var(S)+\sigma_\varepsilon^2}.
\]
We now compute $\Var(S)$. The variance of the first term is
\[
\Var(\bbeta_a^\top\bmu^\top A)=\bbeta_a^\top\bmu^\top\bSigma_A\bmu\,\bbeta_a.
\]
The variance of the remaining two terms is
\[
\Var(\bbeta_a^\top E+\bbeta_z^\top\bX^z)
=
\bbeta_a^\top\bSigma_e\bbeta_a
+\bbeta_z^\top\bSigma_{xz}\bbeta_z
+2\bbeta_a^\top\bSigma_{12}\bbeta_z.
\]
Moreover, because $E$ has mean zero and is independent of $A$, and because $E(\bX^z\mid A)=0$ by construction of Model~\ref{model:3.1}, the term $\bbeta_a^\top\bmu^\top A$ is uncorrelated with $\bbeta_a^\top E+\bbeta_z^\top\bX^z$. Hence
\[
\Var(S)
=
\bbeta_a^\top\bmu^\top\bSigma_A\bmu\,\bbeta_a
+\bbeta_a^\top\bSigma_e\bbeta_a
+\bbeta_z^\top\bSigma_{xz}\bbeta_z
+2\bbeta_a^\top\bSigma_{12}\bbeta_z.
\]
Substituting into the formula for $R_X^2$ yields the stated expression.

For the fairness-adjusted regression, we regress the predictors on $A$ and retain the residuals. Under the model, the residualized version of $\bX^a$ is exactly $E$, while $\bX^z$ is unchanged because its conditional mean does not depend on $A$. Hence the adjusted predictor vector is
\[
U=(E^\top,(\bX^z)^\top)^\top.
\]
The linear projection of $Y$ onto $U$ therefore has signal
\[
S_U=\bbeta_a^\top E+\bbeta_z^\top\bX^z,
\]
whose variance is
\[
\Var(S_U)
=
\bbeta_a^\top\bSigma_e\bbeta_a
+\bbeta_z^\top\bSigma_{xz}\bbeta_z
+2\bbeta_a^\top\bSigma_{12}\bbeta_z.
\]
The total variance of $Y$ is unchanged by residualizing the predictors, so
\[
R_U^2
=
\frac{\Var(S_U)}{\Var(S)+\sigma_\varepsilon^2}.
\]
Substituting the formulas above gives the expression in the theorem.

Finally, the sample versions of the required covariance matrices are root-$n$ consistent under finite fourth moments. Since the map from the covariance parameters to the displayed rational expressions is smooth on the parameter space where the denominator is positive, a delta-method argument yields the $O_p(n^{-1/2})$ remainder terms for both $R_X^2$ and $R_U^2$.
\end{proof}

\subsubsection{Proof of Theorem~\ref{thm:R2-univariate}}
\label{app:proof-R2-univariate}
\begin{proof}
In the binary Gaussian setting, write
\[
X=\mu A+E,
\]
where $A\sim\mathrm{Ber}(p)$, $E\sim N(0,\sigma_e^2)$, and $E$ is independent of $A$. Since $\Var(A)=p(1-p)$, we have
\[
\Var(X)=\Var(\mu A)+\Var(E)=p(1-p)\mu^2+\sigma_e^2.
\]
The outcome model is
\[
Y=\beta_0+\beta_1 X+\varepsilon,
\qquad \varepsilon\sim N(0,\sigma_\varepsilon^2),
\]
with $\varepsilon$ independent of $X$. Therefore the signal variance in the full regression is
\[
\Var(\beta_1 X)=\beta_1^2\Var(X)=\beta_1^2\{p(1-p)\mu^2+\sigma_e^2\},
\]
and the total variance of $Y$ is
\[
\Var(Y)=\beta_1^2\{p(1-p)\mu^2+\sigma_e^2\}+\sigma_\varepsilon^2.
\]
Hence
\[
R_X^2
=
\frac{\beta_1^2\{p(1-p)\mu^2+\sigma_e^2\}}
{\beta_1^2 p(1-p)\mu^2+\beta_1^2\sigma_e^2+\sigma_\varepsilon^2}.
\]

To residualize $X$ with respect to $A$, note that the conditional mean is
\[
E(X\mid A)=\mu A.
\]
Thus the residual is
\[
U=X-E(X\mid A)=E.
\]
The fairness-adjusted regression therefore uses only $U=E$ as the predictor. Its signal variance is
\[
\Var(\beta_1 U)=\beta_1^2\sigma_e^2,
\]
while the total variance of $Y$ remains
\[
\Var(Y)=\beta_1^2\{p(1-p)\mu^2+\sigma_e^2\}+\sigma_\varepsilon^2.
\]
Consequently,
\[
R_U^2
=
\frac{\beta_1^2\sigma_e^2}
{\beta_1^2 p(1-p)\mu^2+\beta_1^2\sigma_e^2+\sigma_\varepsilon^2}.
\]
As in the proof of Theorem~\ref{thm:R2-general}, replacing population moments by their sample counterparts introduces $O_p(n^{-1/2})$ errors under finite fourth-moment conditions, which gives the stated asymptotic expansions.
\end{proof}

\bibliographystyle{elsarticle-harv}
206
\bibliography{references}

\end{document}